\documentclass[11pt]{article}
\usepackage{fullpage}
\usepackage{hyperref}       
\usepackage{url}    
\usepackage{booktabs}       
\usepackage{nicefrac}       
\usepackage{microtype}  
\usepackage{times}   
\usepackage{xcolor,colortbl} 
\usepackage{multirow}
\usepackage{tablefootnote}
\usepackage{subcaption}

\usepackage[style=alphabetic,maxnames=12,maxbibnames=10,maxcitenames=10,maxalphanames=10,giveninits=true,doi=false,url=true]{biblatex}
\newcommand*{\citet}[1]{\AtNextCite{\AtEachCitekey{\defcounter{maxnames}{2}}} \textcite{#1}}

\newcommand*{\citep}[1]{\cite{#1}}

\usepackage{amsmath}
\usepackage{amssymb}
\usepackage{mathtools}
\usepackage{amsthm}
\usepackage{dsfont}
\usepackage{algorithm}
\usepackage{algorithmic}

\usepackage{hyperref}
\usepackage[capitalize,noabbrev]{cleveref}

\theoremstyle{plain}
\newtheorem{theorem}{Theorem}[section]

\newtheorem{lemma}[theorem]{Lemma}
\newtheorem{claim}[theorem]{Claim}
\newtheorem{corollary}[theorem]{Corollary}
\theoremstyle{definition}
\newtheorem{definition}[theorem]{Definition}

\newtheorem{assumption}[theorem]{Assumption}
\theoremstyle{remark}
\newtheorem{remark}[theorem]{Remark}
\theoremstyle{definition}

\usepackage{mathrsfs}
\usepackage{mathnotations}
\usepackage{algorithmic}
\usepackage{subcaption, wrapfig}

\usepackage{bbm}

\usepackage[textsize=tiny]{todonotes}

\usepackage{minitoc}
\usepackage{tikz}
	\usetikzlibrary{positioning}
	\definecolor{DarkGreen}{rgb}{0.2,0.6,0.2}
	\definecolor{DarkRed}{rgb}{0.6,0.2,0.2}
	\definecolor{DarkBlue}{rgb}{0.2,0.2,0.6}
	\definecolor{DarkPurple}{rgb}{0.4,0.2,0.4}   
\hypersetup{
        linktocpage=true,
        colorlinks=true,				
        linkcolor=DarkBlue,				
        citecolor=DarkBlue,				
        urlcolor=DarkBlue,			
    }

\begin{document}

\title{Differentially Private Reward Estimation with Preference Feedback}
\author {
    Sayak Ray Chowdhury\thanks{Equal contribution.} \thanks{Microsoft Research, India. Email: \texttt{t-sayakr@microsoft.com}  } \quad 
    Xingyu Zhou \footnotemark[1] \thanks{Wayne State University, Detroit, USA.  Email: \texttt{xingyu.zhou@wayne.edu}} \quad
    Nagarajan Natarajan\thanks{Microsoft Research, India. Email: \texttt{nagarajn@microsoft.com}  }
}

\date{}

\maketitle

\begin{abstract}
Learning from preference-based feedback has recently gained considerable traction as a promising approach to align generative models with human interests. Instead of relying on numerical rewards, the generative models are trained using reinforcement learning with human feedback (RLHF). These approaches first solicit feedback from human labelers typically in the form of pairwise comparisons between two possible actions, then estimate a reward model using these comparisons, and finally employ a policy based on the estimated reward model. An adversarial attack in any step of the above pipeline might reveal private and sensitive information of human labelers. In this work, we adopt the notion of \emph{label differential privacy} (DP) and focus on the problem of reward estimation from preference-based feedback while protecting privacy of each individual labelers. Specifically, we consider the parametric Bradley-Terry-Luce (BTL) model for such pairwise comparison feedback involving a latent reward parameter $\theta^* \in \Real^d$. Within a standard minimax estimation
framework, we provide tight upper and lower bounds on the error in estimating $\theta^*$ under both \emph{local} and \emph{central} models of DP. We show, for a given privacy budget $\epsilon$ and number of samples $n$, that the additional cost to ensure label-DP under local model is $\Theta \big(\frac{1}{ e^\epsilon-1}\sqrt{\frac{d}{n}}\big)$, while it is $\Theta\big(\frac{\text{poly}(d)}{\epsilon n} \big)$ under the weaker central model. We perform simulations on synthetic data that corroborate these theoretical results.

\end{abstract}

\newpage
\tableofcontents
\newpage

\section{Introduction}


In recent years, the problem of aligning generative models to human preferences has garnered a lot of interest~\citep{christiano2017deep,glaese2022improving,ouyang2022training}. One of the most promising approaches to achieve this is via preference-based reinforcement learning \citep{christiano2017deep}. It has gained considerable attention across multiple application domains such as game playing \citep{macglashan2017interactive}, large language models \citep{ouyang2022training}, and robotics \citep{shin2023benchmarks}.

\textbf{Preference-based learning:} In standard RL, the agent learns to maximize a numerical reward, which she observes from the environment. In the above applications, however, observing appropriate numerical rewards can be challenging, which could significantly affect the performance of the agent. In such cases, it is of common practice to solicit feedback from a human labeler in the form of pairwise comparisons between two possible actions at every state \citep{christiano2017deep}. Notably, the language model application InstructGPT \citep{ouyang2022training} is based on this feedback model. First, the prompts (states) are sampled from a pre-collected datasest, and then, for each prompt,
a pair of responses (actions) are sampled by deploying the pre-trained model. For each prompt, a human labeler provides pairwise preferences over the responses, which are then used to train a reward model by maximum likelihood estimation, or, equivalently, by cross-entropy minimization \citep{christiano2017deep}. Finally, this reward model is used for a downstream policy training (i.e., finetuning the pre-trained model). This complete pipeline forms the basis of preference-based RL, see e.g. \cite{pacchiano2021dueling,chen2022human,zhu2023principled,zhan2023provable}.


\textbf{Privacy (or the lack of it):} One important aspect which is ignored in the aforementioned learning literature is protecting privacy of human labelers. Potentially sensitive information of an individual can be revealed through the collected (pairwise comparisons) feedback in case of an adversarial attack at any stage of the RL pipeline. In fact,
after the emergence of ChatGPT, several instances of privacy breach including that of human labelers have been reported \citep{li2023multi}. Since then, efforts have been made to privately fine-tune large language models \citep{yu2021differentially,behnia2022ew}.

\textbf{Label-Differential Privacy:} In view of this, Differential privacy (DP)~\citep{dwork2008differential} is the most adopted notion to protect the sensitive information of individuals whose preference feedback is used during the model training. 
The prompts (states) are not considered sensitive since they are typically sampled from a pre-collected dataset which is already public knowledge. In this work, we develop new results for privacy (as well as accuracy) of estimators obtained with such potentially sensitive feedback information via the notion of label differential privacy (Label-DP). This notion of label-DP has been studied previously in deep learning ~\citep{ghazi2021deep} and in learning theory in general \citep{chaudhuri2011sample,beimel2013private}. 

We focus on the problem of reward estimation from pairwise preferences while protecting the privacy of individual labelers. Specifically, we consider the parametric Bradley-Terry-Luce (BTL) model for such feedback involving a latent reward parameter $\theta^* \in \Real^d$. We prove upper and lower bounds on the error in estimating $\theta^*$ under \emph{local} (where the learner only observes privatized labels) and \emph{central} models (where the learner has access to the raw non-private data) of DP. 

\textbf{Our contributions:} We summarize our contributions as follows.

\noindent \textbf{(1)} We show that the additional cost in estimation for ensuring $\epsilon$-label-DP under local model is at least $\Omega \big(\frac{1}{ e^\epsilon-1}\sqrt{\frac{d}{n}}\big)$, where $\epsilon$ is a given privacy budget and $n$ is the total number of samples.\\
\textbf{(2)} For the local model, we design an estimator of $\theta^*$ based on the \textit{Randomized Response} (RR) mechanism~\citep{warner1965randomized} that satisfies $\epsilon$ label-DP and achieves a matching upper bound on estimation error. To do so, we design a novel loss function tailored to RR, which de-biases the effect of label randomization and can potentially be of independent interest.\\
\textbf{(3)} For the central model, we show that the additional cost for ensuring $(\epsilon,\delta)$-label-DP under this weaker privacy model is at least $\Omega \big(\frac{1}{\epsilon+\delta}\frac{\sqrt{d}}{n}\big)$ for $\delta \in (0,1)$.\\
\textbf{(4)} Finally, for the central model, we also provide a matching upper bound (in $n$ and $\epsilon$) by designing an estimator of $\theta^*$ based on the classical \emph{objective perturbation} technique with Gaussian privacy noise.

We present numerical simulations on synthetic data to support our theoretical results.


\vspace{-2mm}
\paragraph{Related work.} 
Our work is inspired by a recent study on reward estimation (and offline bandits/RL) under the linearly parametrized BTL model without privacy protection~\citep{zhu2023principled}. Our work introduces label-DP into the same setting and provides sharp results on estimation errors as well as some downstream applications. Both~\citet{zhu2023principled} and our work can be viewed as a generalization of the work on non-private reward estimation under the \emph{tabular} BTL model~\citep{shah2015estimation}, which studies estimation error under both semi-norm and $\ell_2$-norm. Label-DP is first introduced by~\citet{chaudhuri2011sample} for private PAC learners. Recently, it has been leveraged to yield better performance for many practical situations where only labels are sensitive data, relative to standard DP which is an overkill~\citep{ghazi2021deep,malek2021antipodes,esfandiari2022label}. As in~\cite{esfandiari2022label}, we consider label DP under both local and central models. We also remark that our work differs from the vast literature on private logistic regression (or stochastic optimization) (e.g.,~\cite{chaudhuri2008privacy,song2021evading,bassily2014private}) in the performance metrics, i.e., parameter estimation error vs. generalization/excess population risk. See more details and additional related work in Appendix~\ref{app:addRelated}.


\section{Preliminaries}
Let $\cD=(s_i,a^0_i,a^1_i,y_i)_{i=1}^{n}$ be a dataset of $n$ samples, where each sample has a state $s_i \in \cS$ (e.g., prompt given to a language model), two actions $a_i^0, a_i^1 \in \cA$ (e.g., two responses from the language model), and label $y_i \in \{0,1\}$ indicating which action is preferred by humans experts. We assume that the state $s_i$ is first sampled from some fixed distribution $\rho$. The pair of actions $(a_i^0, a_i^1)$ are then sampled from some joint distribution (i.e., a behavior policy) $\mu$ conditioned on $s_i$. Finally, the label $y_i$ is sampled from a Bernoulli distribution conditioned on $(s_i, a_i^0, a_i^1)$, i.e., for $l \in \{0,1\}$,
\begin{align*}
\mathbb{P}_{\theta^*}\!\!\left[y_i\!=\!l |s_i, a_i^0, a_i^1\!\right] \!=\! \frac{\exp(r_{\theta^*}(s_i,a_i^l))}{\exp(r_{\theta^*}(s_i,a_i^0)) \!+\! \exp(r_{\theta^*}(s_i,a_i^1))}.
\end{align*}
Here $r_{\theta^*}(\cdot,\cdot)$ is the reward model parameterized by an unknown parameter $\theta^*$, which we would want to estimate using $\cD$. This model is often called Bradley-Terry-Luce (BTL) model~\citep{bradley1952rank,luce2012individual}. 

In this work, we consider a linear reward model $r_{\theta^*}(s,a) = \phi(s,a)^\top\theta^*$, where $\phi: \cS \times \cA \to \Real^d$ is some known and fixed feature map. For instance, such a $\phi$ can be constructed by
removing the last layer of a pre-trained language model, and in that case, $\theta^*$ correspond to the weights of the last layer. With this model, one can equivalently write the probability of sampling $y_i = 1$ given $(s_i, a_i^0, a_i^1)$ as 
\begin{align*}
   \!\mathbb{P}_{\theta^*}\!\!\left[ y_i=1 |s_i, a_i^0, a_i^1\right]= \sigma\!\left(\!\left(\phi(s_i,a_i^1)\!-\! \phi(s_i,a_i^0)\right)^\top \theta^* \!\right) ,
\end{align*}
where $\sigma(z) \!=\! \frac{1}{1+ e^{-z}}$ is the sigmoid function. We let $x_i = \phi(s_i,a_i^1)\!-\! \phi(s_i,a_i^0)$ denote the differential feature of actions $a_i^1$ and $a_i^0$ at state $s_i$. This lets us denote, for any $\theta \in \Real^d$, the predicted probabilities of a label $y_i$ given $x_i$ as (we omit dependence on $\theta$ for brevity)
\begin{align}\label{eq:gen-np}
p_{i,1} \!:=\!\mathbb{P}_\theta \left[y_i\!=\!1 |x_i\right]\!=\! \sigma(x_i^\top \theta)~,\, p_{i,0}\!:= \! 1\!-\! p_{i,1}~.
\end{align}
We make the following assumption which is standard in the literature \citep{shah2015estimation,zhu2023principled}.
\begin{assumption}[Boundedness]
\label{ass:bound}
(a) $\theta^*$ lies in the set $\Theta_B = \{\theta \in \Real^d | \inner{\mathbf{1}}{\theta} = 0, \norm{\theta} \le B \}$. The condition $\inner{\mathbf{1}}{\theta} = 0$ ensures identifiability of $\theta^*$. (b) Features are bounded, i.e., $\norms{\phi(s,a)} \le L$, $ \forall (s,a)$.
\end{assumption}
Now, we recall the
notion of differential privacy \citep{dwork2008differential}. Roughly, it ensures that the output of an algorithm $\cM$ 
operating on a dataset $\cD$ doesn't change much if we change a single example in $\cD$. 
In this paper, we adopt the notion of \emph{label DP} \citep{ghazi2021deep} to protect sensitive information that lies in preference-based feedback $y_i$. This is motivated by the fact that in most applications, the data $(s_i, a_i^0, a_i^1)$ presented to the human annotator is public (or pre-collected) while the feedback $y_i \in \lbrace 0,1 \rbrace$ indicates her personal preference, which needs to be protected. 
\begin{definition}[Label DP in Central Model]
\label{def:central-label}
    Let $\epsilon \ge 0, \delta \in (0,1]$. A randomized algorithm $\cM$ is said to be $(\epsilon,\delta)$-label  differentially private in central model if for any two datasets $\cD$ and $\cD'$ that differ in the \emph{label} of a single sample and for any subset $S$ in the range of $\cM$, 
    it holds that 
    \begin{align*}
        \prob{\cM(\cD) \in S} \le e^{\epsilon}\cdot \prob{\cM(\cD') \in S} + \delta.
    \end{align*}
    If $\delta = 0$, $\cM$ is said to be $\epsilon$-label DP. We will simply call it central label DP in the following.
\end{definition}
For our specific reward estimation problem, Definition~\ref{def:central-label} roughly means that any single change of feedback label will not change the final estimator too much.

The central DP model assumes that the learning agent $\cA$ has access to preference feedback given by human labelers in the clear-text. In some applications, however, the individual labelers might not be willing to share their feedback in the clear-text. This motivates us to consider label DP in the local model, where each feedback $y_i$, before being observed by the agent, is first privatized by some local randomizer $\cR$ at each labeler, which is formally defined as follows.

\begin{definition}[Label DP in Local Model]
\label{def:local DP}
   If each label is first privatized by a local randomizer $\cR$, which satisfies for any $y, y'$ and any subset $S$ in the range of $\cR$, it holds that
   \begin{align*}
        \prob{\cR(y) \in S} \le e^{\epsilon}\cdot \prob{\cR(y') \in S} + \delta ,
    \end{align*}
    then, we say $\cR$ is an $(\epsilon,\delta)$-label differentially private local randomizer, and the entire algorithm (e.g., estimator) that operates with the randomized labels is said to satisfy local label DP. 
\end{definition}
\begin{remark}
    Note that for central label DP, the privacy burden lies in the central agent while for local label DP, the privacy protection relies on local randomizer $\cR$. By post-processing of DP~\citep{dwork2008differential}, an algorithm that satisfies local label DP also satisfies central local DP. Thus, in the following, we will first focus on the stronger local model of privacy.
\end{remark}



The rest of the paper is organized as follows. In the next two sections, we will focus on local label DP and present the lower bound and upper bounds for the estimation error, respectively. In Section~\ref{sec:central}, we turn to the central label DP and also present corresponding lower and upper bounds. 

\section{Local Model: Lower Bound on Estimation Error}\label{sec:lower}
In this section, we present the lower bounds on the estimation error under local label DP (cf. Definition~\ref{def:local DP}).
Let $\Sigma_{\cD} \!:=\! \frac{1}{n}\sum_{i=1}^n x_ix_i^\top$ denote the sample covariance matrix of differential features $x_i = \phi(s_i,a_i^1)\!-\! \phi(s_i,a_i^0)$. Then, for any $\lambda >0$, we have the following lower bound on the estimation error in the semi-norm $\norm{\cdot}_{\Sigma_\cD+\lambda I}$. 
\begin{theorem}[Semi-norm lower bound]
\label{thm:semi-lower}
    For a large enough $n$, any estimator $\hat{\theta}$ based on $n$ samples from the BTL model that satisfies $\epsilon$-label DP in the local model has estimation error in semi-norm lower bounded as 
    \begin{align*}
    \ex{ \norm{\hat{\theta} - \theta^*}_{\Sigma_{\cD} + \lambda I}^2} \ge \Omega\left(\frac{d}{n} + \frac{d}{(e^{\epsilon}-1)^2 n}\right).
\end{align*}
\end{theorem}
\citet{shah2015estimation} shows that (squared) error of estimation under the tabular BTL model of pairwise comparisons is at least $\Omega(d/n)$ without any privacy constraint. In comparison, we pay an additional $\Omega \big(\frac{d}{ (e^\epsilon-1)^2 n}\big)$ error in estimation in order to ensure $\epsilon$-label DP. We suffer a similar privacy cost while bounding estimation error in $\ell_2$-norm too. The result is formally stated below  
\begin{theorem}[$\ell_2$-norm lower bound]
\label{thm:l2-lower}
Under the same hypothesis of Theorem~\ref{thm:semi-lower}, the estimation error of $\hat \theta$ in $\ell_2$-norm is lower bounded as 
    \begin{align*}
     \ex{ \norm{\hat{\theta} - \theta^*}^2 }\ge \Omega\left( \frac{d}{L^2} \cdot \left(\frac{d}{n} + \frac{d}{(e^{\epsilon}-1)^2 n}\right) \right).
\end{align*}
\end{theorem}
Note that the lower bound in $\ell_2$ norm is an $\Omega(d)$ multiplicative factor higher than the one in semi-norm (when $L = \Theta(1)$). A similar comparative behavior holds for non-private lower bounds too \citep{shah2015estimation}. Moreover, if the differential features $x_i$ is distributed according to a standard Gaussian, then we have $L = O(\sqrt{d})$. In this case, the first term in $\ell_2$-norm lower bound reduces to $\Omega(d/n)$, which recovers the mean-squared-error (MSE) lower bound under Gaussian design without any privacy considerations~\citep{chen2016bayes,hsu2023sample}.

\textbf{Proof summary of lower bounds.} For both lower bounds, we leverage the classic reduction from estimation to testing. In particular, for the $\ell_2$ norm, we apply a variant of the (private) version of Assouad’s lemma (cf.~\cite{yu1997assouad}) by constructing a hypercube over the underlying parameter space. On the other hand, for the semi-norm case, it is somewhat difficult to construct a hypercube. Instead, we turn to (private) version of Fano's lemma, which only requires a packing (in terms of semi-norm). This can be achieved by Varshamov–Gilbert’s bound (cf.~\cite{guntuboyina2011lower}) and vector rotations. For the privacy parts in both bounds, we leverage strong data processing inequality under local DP (cf.~\cite{duchi2018minimax}). The complete proofs for both results are presented in Appendix~\ref{app:lower}.


\section{Local Model: Upper Bound on Estimation Error}\label{sec:local}

In this section, we discuss private estimators of the unknown parameter $\theta^*$ and develop a series of results that answer the following questions.\\
\textbf{(1)} \textit{Is the standard MLE estimator useful under the Randomized Response model, and in what privacy regime?}\\
\textbf{(2)} \textit{Can we design an estimator for all privacy regimes that achieves the same order of estimation as in the lower bound?} \\ 
\textbf{(3)} \textit{How do we compute the estimator efficiently?}\\
\textbf{(4)} \textit{Can we extend the ideas to other popular preference feedback models such as \emph{Thurstone} and \emph{Placket-Luce}?}\\
We first describe the Randomized Response (RR) mechanism \citep{warner1965randomized}, which we use to guarantee local label DP. 

\textbf{Randomized Response.}
Let $\epsilon \ge 0$ be 
the privacy budget and $y \in \lbrace 0,1\rbrace$ be the true label. When queried the value of $y$, the RR mechanism outputs $\widetilde y$, which is randomly sampled from the probability distribution
\begin{align}\label{eq:RR}
   \prob{\widetilde{y} = y} = \frac{e^{\epsilon}}{1+e^{\epsilon}} =\sigma(\epsilon),\,\, \prob{\widetilde{y} \neq y} = 1-\sigma(\epsilon)~. 
\end{align}
It is well-known that RR is $\epsilon$-DP \citep{dwork2008differential}. 
In the following, we will use RR as $\cR$ in Definition~\ref{def:local DP} to achieve label DP in the local model.

We start with a simple maximum likelihood estimator (MLE), which will help us develop intuition for a comparatively complex but a better estimator.  

\subsection{The Maximum Likelihood Estimator}
 For any $\theta \!\in\! \Real^d$, \eqref{eq:gen-np} and \eqref{eq:RR} together define predicted probabilities of a randomized label $\widetilde y_i$ given $x_i$ as
\begin{equation*}\label{eq:gen-p}
\begin{split}
  \widetilde p_{i,1}&\!=\!\sigma(x_i^\top\theta )\sigma(\epsilon) \!+\! (1\!-\!\sigma(x_i^\top\theta)) (1\!-\!\sigma(\epsilon))~,\\
 \widetilde p_{i,0}&\!=\! (1\!-\!\sigma(x_i^\top\theta ))\sigma(\epsilon) \!+\! \sigma(x_i^\top\theta) (1\!-\!\sigma(\epsilon))~. 
\end{split}
\end{equation*}
With $n$ such pairs of features and randomized labels $(x_i,\widetilde y_i)_{i=1}^n$, the private MLE $\widetilde \theta_{\text{MLE-RR}}$ aims to minimize the negative log-likelihood
\begin{align}\label{eq:mle_rr}
\!\!\widetilde l_{\cD,\epsilon}(\theta)\!=\!-\!\sum_{i=1}^n\! \Big[ \!\mathds{1}(\widetilde y_i\!=\!1)\log \widetilde p_{i,1}\!+\! \mathds{1}(\widetilde y_i\!=\!0)\log \widetilde p_{i,0}\!\Big].
\end{align}
As mentioned before, $\widetilde\theta_{\texttt{MLE}}$ is $\epsilon$-label DP for any $\epsilon \ge 0$. Recall that $\Sigma_{\cD} \!=\! \frac{1}{n}\sum_{i=1}^n x_ix_i^\top$ denotes the sample covariance matrix of differential features and let $\gamma$ be a constant such that $\sigma'(x^\top \theta) \ge \gamma$ for all $\theta \in \Theta_B$ and for all features $x$. Under Assumption~\ref{ass:bound}, $\gamma= \frac{1}{2 + e^{-2LB} + e^{2LB}}$ satisfies this condition. Then, we have the following estimation error bound for $\widetilde\theta_{\texttt{MLE}}$.


\begin{theorem}[Estimation error of MLE]\label{thm:mle_rr}
    Fix $\alpha \in (0,1), \epsilon > 2LB, \lambda >0$. Then, under Assumption~\ref{ass:bound}, with probability at least $1-\alpha$, we have
    \begin{align*}
\norms{\widetilde\theta_{\texttt{MLE}} \!-\! \theta^*}_{\Sigma_{\cD}\!+\!\lambda I}  \!\le\! \frac{C}{\gamma}\frac{ e^{\epsilon+2LB}\!+\!1}{ e^{\epsilon-2LB}\!-\!1}\!\sqrt{\!\frac{d\!+\!\log(1/\alpha)}{n}}\!+\!\sqrt{\lambda}B,
\end{align*}
where $C$ is some absolute constant.
\end{theorem}

The above error bound of $\widetilde{\theta}_{\text{MLE}}$ holds only when the privacy budget is higher than a certain threshold (which depends on the norm of $\theta^*$ and $\phi$), i.e., when $\epsilon > 2LB$, thus limiting its applicability only to lower privacy regimes (since a high value of $\epsilon$ implies a low level of privacy). This is due to the fact that $\widetilde{\theta}_{\text{MLE}}$ minimizes a noisy objective \eqref{eq:mle_rr}, that is strongly convex in the semi-norm $\norm{\cdot}_{\Sigma_\cD}$ only if $\epsilon > 2LB$, which is a crucial step in bounding the estimation error (see Appendix~\ref{proof:mle-rr} for details). Instead, we want an objective that is strongly convex in the semi-norm for all privacy levels $\epsilon > 0$. This leads us to an estimator that is specifically tailored to the RR mechanism.

\subsection{An Estimator Tailored to RR}

For any $\theta \in \Real^d$, the logits (log-odds) of the probability that the clear-text label $y_i=1$ given $x_i$ is 
\begin{align*}
    \text{logit}(p_{i,1}) \!=\! \log\frac{p_{i,1}}{p_{i,0}}\!=\!\log\frac{\sigma(x_i^\top \theta)}{1-\sigma(x_i^\top \theta)},
\end{align*}
where the same for randomized label $\widetilde y_i=1$ is
\begin{align*}
    \text{logit}(\widetilde p_{i,1})\! =\! \log \frac{\sigma(x_i^\top \theta)\sigma(\epsilon) \!+\! (1\!-\!\sigma(x_i^\top \theta)) (1\!-\!\sigma(\epsilon))}{(1\!-\!\sigma(x_i^\top \theta))\sigma(\epsilon)\! +\! \sigma(x_i^\top \theta) (1\!-\!\sigma(\epsilon))}~.
\end{align*}
It holds that (see Appendix~\ref{app:local} for details)
\begin{align*}
   \text{logit}(\widetilde p_{i,1}) \le \sigma(\epsilon)\cdot  \text{logit}( p_{i,1})\, \text{if}\,\, p_{i,1} \ge p_{i,0}~,\\
   \text{logit}(\widetilde p_{i,0}) \le \sigma(\epsilon)\cdot  \text{logit}( p_{i,0})\, \text{if}\,\, p_{i,0} \ge p_{i,1}~.
\end{align*}
Since $\sigma(\epsilon) \!\in\! (1/2,1)$ for any $\epsilon > 0$, this implies that whenever $y_i$ is more likely to occur than $1-y_i$ in the clear-text, the log-odds of predicting $y_i$ under $\epsilon$-randomization given by \eqref{eq:RR} is at most $\sigma(\epsilon)$-th fraction of the corresponding log-odds in the clear-text. This makes the objective \eqref{eq:mle_rr} ill-suited for obtaining a tight estimator for $\theta^*$ under randomization of labels. 

Essentially, we want to design an objective (or, equivalently a loss function) so that the log-odds of predictions under randomization is same as the log-odds in the clear-text. The following loss achieves this:
\begin{align}\label{eq:robust_rr}
\!\!\hat l_{\cD,\epsilon}(\theta)\!=\!-\!\!\sum_{i=1}^n \!\!\Big[\! \mathds{1}(\widetilde y_i\!=\!1)\log \hat p_{i,1}\!+\! \mathds{1}(\widetilde y_i\!=\!0)\log \hat p_{i,0}\Big],
\end{align}
where we define, for any $\theta \!\in\! \Real^d$, the predicted \emph{scores} of each randomized label $\widetilde y_i$ given $x_i$ as
\begin{align}\label{eq:predcit_score}
 \! \hat p_{i,1}\!=\!\!\frac{ \sigma(x_i^\top \theta)^{\sigma(\epsilon)}}{ (1\!-\!\sigma(x_i^\top \theta))^{(1\!-\!\sigma(\epsilon))}},\,\, \hat p_{i,0}\!=\!\!\frac{ (1\!-\!\sigma(x_i^\top \theta))^{\sigma(\epsilon)}} {\sigma(x_i^\top \theta)^{(1\!-\!\sigma(\epsilon))}}. 
\end{align}
Although $\hat p_{i,1}$ and $\hat p_{i,0}$ are not probabilities, these satisfy our desired property:
\begin{align*}
    \log\frac{\hat p_{i,1}}{\hat p_{i,0}}\!=\!\log\frac{\sigma(x_i^\top \theta)}{1-\sigma(x_i^\top \theta)} \!=\! \text{logit}(p_{i,1})~. 
\end{align*}
Hence the loss function $\hat l_{\cD,\epsilon}(\theta)$ essentially de-biases the effect of randomization.
This, in turn, yields that $\hat l_{\cD,\epsilon}(\theta)$ is $\gamma(2\sigma(\epsilon)\!-\!1)$ strongly convex in the semi-norm $\norm{\cdot}_{\Sigma_\cD}$ for all $\theta \in \Theta_B$, and importantly, it holds for any $\epsilon > 0$. This helps us
obtain an estimator for $\theta^*$ with error bound for all privacy levels $\epsilon > 0$, defined as
\begin{align}\label{eq:estimator}
 \hat\theta_{\texttt{RR}}\in \argmin\nolimits_{\theta \in \Theta_B} \hat l_{\cD,\epsilon}(\theta) ~.  
\end{align}
$\hat \theta_{\text{RR}}$ satisfies $\epsilon$-label DP due to RR and post-processing of DP. Now, for any constant $\lambda > 0$, we have the following estimation error bound for $\hat \theta_{\text{RR}}$. Proof of this result is deferred to Appendix~\ref{proof:robust-rr}.
\begin{theorem}[Estimation error of $\hat\theta_{\texttt{RR}}$]\label{thm:robust_rr}
Fix $\alpha \in (0,1), \epsilon > 0, \lambda >  0$. Then, under Assumption~\ref{ass:bound}, with probability at least $1-\alpha$, we have
\begin{align}\label{eq:robust_rr_bound}
 \!\!\! \norms{\hat\theta_{\texttt{RR}} \!-\! \theta^*}_{\Sigma_{\cD}\!+\!\lambda I}  \!\le\! \frac{C}{\gamma}\frac{ e^{\epsilon}\!+\!1}{ e^{\epsilon}\!-\!1}\!\sqrt{\!\frac{d\!+\!\log(1\!/\!\alpha)}{n}}\!+\! C'\!\sqrt{\lambda} B,
\end{align}
where $\gamma\!=\! \frac{1}{2 + e^{-2LB} + e^{2LB}}$, $C,C'$ are absolute constants.
\end{theorem}
\textbf{Cost of Privacy.} We compare the error of our private estimator $\hat{\theta}_{\texttt{RR}}$ with that of the clear-text (i.e., non-private) estimator $\theta_{\texttt{MLE}}$, which minimizes the following non-private negative log-likelihood
\begin{align}\label{eq:mle}
\!l_{\cD}(\theta)\!=\!-\!\sum_{i=1}^n\! \Big[\! \mathds{1}(y_i\!=\!1)\log p_{i,1}\!+\! \mathds{1}(y_i\!=\!0)\log p_{i,0}\!\Big].
\end{align}
As shown in~\cite{zhu2023principled}, $\theta_{\texttt{MLE}}$ achieves an estimation error of $ O\big( \sqrt{d/n}\big)$ in semi-norm. Comparing this with Theorem~\ref{thm:robust_rr}, we observe that the cost of ensuring label DP for $\hat{\theta}_{\texttt{RR}}$ is of the order $O \left(\frac{1}{ e^\epsilon-1}\sqrt{\frac{d}{n}}\right)$, which almost matches lower bound  discussed below.

\textbf{Comparison with Lower Bound.} We now compare the upper bound in Theorem~\ref{thm:robust_rr} with the semi-norm lower bound in Theorem~\ref{thm:semi-lower}. Setting $\lambda = \left(\frac{e^{\epsilon}+1}{e^{\epsilon}-1}\right)^2 \frac{d + \log(1/\alpha)}{B^2 \gamma^2 n}$, we see that the upper bound matches the lower bound up to a factor of $O(1/\gamma) \approx e^{LB}$. Hence, if both $L = O(1)$ and $B = O(1)$, the bounds are tight up to a constant factor. 


\textbf{Applications in Contextual Bandits.}
In applications such as offline linear contextual bandits \citep{li2022pessimism}, this bound can then be used to learn a downstream pessimistic policy 
\begin{align}\label{eq:pessimism}
    \hat \pi_\Theta = \argmax_{\pi \in \Pi}\inf_{\theta \in \Theta} \mathbb{E}_{s \sim \rho}\left[\phi(s,\pi(s))^\top \theta\right].
\end{align}
Here $\Pi$ is the set of all action selection policies $\pi:\cS \to \cA$ and $\Theta$ is a high-probability confidence set for $\theta^*$, i.e., it is a set of all $\theta \in \Theta_B$ that satisfies \eqref{eq:robust_rr_bound}.
Similar to \cite{li2022pessimism}, one can show that this pessimistic policy achieves a \emph{sub-optimality gap} of $O\Big(\!\frac{L}{\gamma}\frac{ e^\epsilon+1}{ e^\epsilon-1}\sqrt{\frac{d}{n}}\norm{\Sigma_\cD\!+\!\lambda I}^{-1/2}\!+\sqrt{\lambda} LB \Big)$ in high probability for any $\lambda > 0$ while guaranteeing label DP.


\subsection{Efficient Computation via SGD}

It is evident that computing the exact minimizer $\hat\theta_{\texttt{RR}}$ in~\eqref{eq:robust_rr} is impractical in practice -- an issue shared by the non-private estimator $\theta_{\texttt{MLE}}$ of \citet{zhu2023principled} as well. Note that even if an approximate solution is allowed, it still requires solving the optimization problem up to a certain accuracy level so as to preserve the same estimation error bound.  This motivates us to consider the (one-pass) SGD algorithm, which iterates over each sample once. In particular,
we replace \eqref{eq:estimator}
by a sequential update rule: 
\begin{align}\label{eq:estimate_SGD}
    \hat \theta_1 = 0~,\,\,\hat\theta_{t+1} = \Pi_{\Theta_B}\Big(\hat\theta_{t} - \eta_t \hat{g}_{t}\Big),\,\, 1\le t \le n~.
\end{align}
Here $\Pi_{\Theta_B}$ is a projection operator onto the set $\Theta_B$, $\eta_t$ is 
a suitable learning rate and $\hat g_{t} = -\nabla_{\hat\theta_t} \log \hat p_{t,\widetilde y_t}$ is the (negative) gradient of the log-predicted score of randomized label $\widetilde y_t$ computed at current estimate $\hat\theta_t$, where $\hat p_{t,y}, y \in {0,1}$ is given by \eqref{eq:predcit_score}. We denote the estimate after $n$ iterations as $\hat\theta_{\texttt{SGD-RR}}$.

Although the error bound under semi-norm as proved in Theorem~\ref{thm:robust_rr} does not hold for this
SGD variant,
we can bound the estimation error of $\hat\theta_{\texttt{SGD-RR}}$ in $\ell_2$-norm. To begin with, we note that Theorem~\ref{thm:robust_rr} implies a bound on the estimation error of $\hat{\theta}_{\texttt{RR}}$ in $\ell_2$-norm. 

\begin{corollary}
Under the same hypothesis of Theorem~\ref{thm:robust_rr}, we have, with probability at least $1-\alpha$,
\begin{align*}
  \norms{\hat{\theta}_{\texttt{RR}} \!-\! \theta^*}_2  \le \frac{C}{\gamma\sqrt{\lambda_{\min}(\Sigma_\cD)}}\frac{ e^{\epsilon}\!+\!1}{ e^{\epsilon}\!-\!1}\!\sqrt{\frac{d\!+\!\log(1/\alpha)}{n}}~.
\end{align*}
where $\lambda_{\min}(\Sigma_\cD)$ is the minimum eigenvalue of $\Sigma_\cD$.
\end{corollary}
This bound is non-trivial only if assume that sample covariance matrix $\Sigma_{\cD}$ is positive definite.
One can also relax this assumption to the population covariance matrix of differential state-action features $x = \phi(s, a^1) - \phi(s, a^0)$, defined as
\begin{align*}
    \Sigma = \mathbb{E}_{s \sim \rho(\cdot), (a^0,a^1)\sim \mu(\cdot | s)} \left[ x x^{\top} \right]~.
\end{align*}
\begin{assumption}[Coverage of feature space]
\label{ass:cov}
The data distributions $\rho,\mu$ are such that $\lambda_{\min}(\Sigma) \ge \kappa$ for some constant $\kappa \!>\! 0$.
\end{assumption}
This is essentially a coverage assumption on the state-action feature space, which is standard in offline bandits and RL \citep{yin2022near}. The next result bounds the estimation error of  $\hat\theta_{\texttt{SGD-RR}}$ in $\ell_2$-norm.
\begin{theorem}[Estimation error of $\hat\theta_{\texttt{SGD-RR}}$]
\label{thm:sgd_rr}
    Fix $\alpha \in (0,1/e)$ and $\epsilon \ge 0$. Then, under Assumptions~\ref{ass:bound} and~\ref{ass:cov} and setting $\eta_t = \frac{1}{\gamma \kappa}$, we have, with probability at least $1-\alpha$, 
    \begin{align*}
\norm{\hat\theta_{\texttt{SGD-RR}} \!-\! \theta^*}_2 \!\le\! C \!\cdot\! \frac{L}{\gamma \kappa} \!\cdot\!\frac{e^{\epsilon}\!+\!1}{e^{\epsilon}\!-\!1}\!\sqrt{ \frac{\log\left(\log(n)/\alpha\right)}{ n}},
    \end{align*}
where $\gamma = \frac{1}{2 + e^{-2LB} + e^{2LB}}$, $C$ is an absolute constant.
\end{theorem}
The complete algorithm and proof of the theorem is deferred to Appendix~\ref{proof:sgd-rr}. In fact, we prove a stronger and general result than Theorem~\ref{thm:sgd_rr} by bounding the estimator error uniformly for all intermediate parameter estimates $\hat\theta_{t+1}$, $1\le t \le n$, with $\sqrt{1/n}$ replaced by $\sqrt{1/t}$ in the bound. The $\log\log n$ term is the (minimal) cost to ensure uniform concentration over all $t \le n$. The bound for $\hat\theta_{\texttt{SGD-RR}}$ follows by setting $t\!=\!n$ in the general result. The key idea behind this result is to show the gradient $\hat g_t$ in the SGD update \eqref{eq:estimate_SGD} is an unbiased estimate of the gradient (except some scaling) in the clear-text $g_t = -\nabla_{\hat\theta_t} \log  p_{t,y_t}$, where $p_{t,y}, y \in \lbrace 0,1\rbrace$ denotes the probability of observing $y$ at round $t$, see \eqref{eq:gen-np}.
Specifically, we have
\begin{align*}
        \hat{g}_t = \frac{\sum_{y\in \lbrace 0,1\rbrace} \nabla_{\hat\theta_t} \log  p_{t,y}}{e^{\epsilon} + 1}- \nabla_{\hat\theta_t} \log  p_{t,\widetilde y_t} ~,
    \end{align*}
which, in turn, gives $\mathbb{E}\left [\hat g_t |x_t,y_t,\hat\theta_t\right] \!=\! \left(2\sigma(\epsilon)-1\right)g_t$, where the expectation is over the $\epsilon$-randomization of clear-text label $y_t$ given by \eqref{eq:RR}. This, along with the coverage assumption and the fact that $\sigma'(x_t^\top \theta) \ge \gamma$ for all $\theta \in \Theta_B$ help us achieve the desired error bound.

    \textbf{Comparison with Semi-norm Bound.}
     The main difference compared to the semi-norm bound in Theorem \ref{thm:robust_rr} is the inverse dependence on coverage parameter $\kappa$ -- estimation error increases as $\kappa$ decreases. Another apparent difference is the dependence (or the lack of it) on the feature dimension $d$ in the error bound.
     However, $\kappa$ is a problem dependent quantity. It depends implicitly on the dimension $d$ of feature space \citep{wang2020statistical}, thereby capturing the dependence of error bound on $d$. For example,
     since $\norm{x} \le L$, we have $\kappa = O(L^2/d)$ under Assumption~\ref{ass:bound}. In the best case when $\kappa = \Theta(L^2/d)$, the error bound in $\ell_2$-norm scales as $\widetilde O\left(\frac{1}{\gamma}\frac{e^\epsilon+1}{e^\epsilon -1}\frac{d}{\sqrt{n}} \right)$, which is $\sqrt{d}$ factor higher than that in semi-norm $\norm{\cdot}_{\Sigma_\cD}$.
     Finally, due to the coverage assumption, instead of employing a pessimistic policy as in~\eqref{eq:pessimism} for a downstream offline contextual bandit task, we can design a greedy (plug-in) policy $
    \hat \pi_{\text{Greedy}}(s) = \argmax_{a \in \cA} \phi(s,a)^\top \hat\theta_{\texttt{SGD-RR}}$,
which achieves a \emph{sub-optimality gap} of $\widetilde O\Big(\!\frac{L^2}{\gamma\kappa}\frac{ e^\epsilon+1}{ e^\epsilon-1}\frac{1}{\sqrt{n}} \Big)$ in high-probability while ensuring label-DP.

\textbf{Comparison with Lower Bound.} We now compare the upper bound in Theorem~\ref{thm:sgd_rr} with the $\ell_2$-norm lower bound  in Theorem~\ref{thm:l2-lower}. First, we note that $\kappa = O(L^2/d)$ under Assumption~\ref{ass:bound}. That is, in the best case when $\kappa = \Theta(L^2/d)$, the upper bound in Theorem~\ref{thm:sgd_rr} becomes $\widetilde{O}\left(\frac{d}{L\gamma\sqrt{n}}\frac{e^{\epsilon}+1}{e^{\epsilon}-1}\right)$. This matches the lower bound up to a factor of $O(e^{LB})$. Hence, similar to the semi-norm bounds, if $L, B$ are $\Theta(1)$, the $\ell_2$-norm bounds are also tight up to a constant factor. 

\begin{remark}
A similar SGD update as \eqref{eq:estimate_SGD} is used in~\citet{ghazi2021deep} in the context of private stochastic convex optimization (SCO). 
They bound the excess population risk of the SGD estimate in expectation under the clear-text distribution $(x_t,y_t) \sim P$. \citet{natarajan2013learning} consider a similar objective function as \eqref{eq:robust_rr} in the context of binary classification with noisy labels. They obtain a classifier by minimizing the empirical risk on the noisy samples $(x_t,\widetilde y_t)_{t=1}^n$ (which is equivalent to maximizing \eqref{eq:robust_rr}) and bound its excess population risk under the clear-text distribution. In contrast to both works,
we aim to bound the estimation error of the parameter estimate in high probability, which brings additional challenges in the analysis.  
\end{remark}

\subsection{Extensions to Other Preference Models}

\paragraph{Thurstone Model.}
 Our result can be extended to any pairwise comparison model of the form
 \begin{align*}
  \mathbb{P}_{\theta^*} \left[y_i\!=\!1 |s_i,a_i^0,a_i^1\right]\!=\! F(x_i^\top \theta^*) 
\end{align*}
if $F$ satisifies following two properties: (i) $F$ is an even function, i.e., $F(z)=1-F(-z)$ for all $z$ and (ii) $F$ is strongly log-concave in an interval around $z=0$, i.e., there is a curvature parameter $\gamma > 0$ and a range parameter $c > 0$ such that
\begin{align*}
    \frac{d^2}{dz^2}\left(-\log(F(z))\right) \ge \gamma \quad \forall \, z \in [-c,c]~.
\end{align*}
For the BTL model, where $F$ is specified by the sigmoid function, these properties hold for $c=2LB$ and $\gamma = \frac{1}{2 + \exp(-2LB) + \exp(2LB)}$ under Assumption~\ref{ass:bound}. 

In the Thurstone model \citep{thurstone1927law}, each label $y_i \in \lbrace 0,1 \rbrace$ is sampled from the conditional distribution 
\begin{align*}
  \mathbb{P}_{\theta^*} \left[y_i\!=\!1 |s_i,a_i^0,a_i^1\right]\!=\! \Phi(x_i^\top \theta^*)~, 
\end{align*}
where $\Phi$ is the CDF of standard Gaussian distribution. It holds that $\Phi$ is strongly
log-concave for all $\theta \in \Theta_B$ under Assumption~\ref{ass:bound} \citep{tsukida2011analyze}. Hence, a similar error bound as Theorem~\ref{thm:sgd_rr} for the SGD-based estimator holds under Thurstone model with a proper choice of the curvature parameter $\gamma$.

\paragraph{Placket-Luce Model.}
One practical extension of our results is to privately learn the reward parameter $\theta^*$ from $K$-wise comparisons between actions, which is captured by the Placket-Luce (PL) model \citep{plackett1975analysis,luce2012individual}. Let $s$ be a state and $a_1,\ldots,a_K$ be $K$ actions to be compared at that state. Let the label/preference feedback $y \in \lbrace 1,2,\ldots,K\rbrace$ indicates which action is most preferred by human labeler.\footnote{We differ here from the standard PL model, where human labeler outputs the entire ranking between $K$ actions.} 
Under the Placket-Luce model, the label $y$ is sampled according to the probability distribution
\begin{align}\label{eq:PL}
\!\!\mathbb{P}_{\theta^*}\!\!\left[y\!=\!k |s, a_1,\ldots, a_K\right] \!=\! \frac{\exp(\phi(s,a_k)^\top \theta^*)}{\sum_{j=1}^{K}\exp(\phi(s,a_j)^\top \theta^*) }.
\end{align}
In this case, label-DP is ensured by employing the $K$-Randomized Response (K-RR) mechanism, which, when queried the value of $y$, outputs $\widetilde y$ that is randomly sampled from the probability distribution:
\begin{align}\label{eq:KRR}
   \!\prob{\widetilde{y} \!=\! y} \!=\! \frac{e^{\epsilon}}{e^{\epsilon}\!\!+\!\!K\!\!-\!\!1},\, \prob{\widetilde{y} \!=\! y'} \!=\! \frac{1}{e^{\epsilon}\!\!+\!\!K\!\!-\!\!1}\forall y'\!\neq\! y~. 
\end{align} 
Given $n$ samples $(s_t,a_{t,1},\ldots a_{t,k},y_t)_{t=1}^n$, we estimate $\theta^*$ using the SGD update of~\eqref{eq:estimate_SGD}, with the gradient 
\begin{align*}
        \hat{g}_t = \frac{\sum_{y=1}^K  \nabla_{\hat\theta_t} \log  p_{t,y}}{e^{\epsilon} +K- 1}- \nabla_{\hat\theta_t} \log  p_{t,\widetilde y_t} ~,
\end{align*}
where $p_{t,y}, y \in [K]$ is the probability of observing $y$ at round $t$, see \eqref{eq:PL}. 



Let $x_{i,j}=\phi(s,a_i)-\phi(s,a_j)$ be the feature difference between actions $a_i$ and $a_j$ at state
and $\Sigma_{i,j}= \mathbb{E}[x_{i,j}x_{i,j}^\top]$ be the corresponding population covariance matrix. Assume there exists a coverage parameter $\kappa >0$ such that $\Sigma_{ij} \ge \kappa$ for all pair of actions $(a_i, a_j)$. Then, similar to Theorem~\ref{thm:sgd_rr}, we can prove an error bound for this SGD-based estimator with K-RR, denoted by $\hat\theta_{\texttt{SGD-KRR}}$. Specifically, we have
\begin{align*}
        \norm{\hat\theta_{\texttt{SGD-KRR}} \!-\! \theta^*}_2 \!=\! \widetilde O\left( \!\cdot\! \frac{L}{\gamma \kappa} \!\cdot\!\frac{e^{\epsilon}\!+K\!-\!1}{e^{\epsilon}\!-\!1}\! \frac{1}{ \sqrt{n}}\right)
    \end{align*}
with high probability, where $\gamma=\frac{1}{e^{4LB}}$. See Appendix~\ref{proof:PL} for a precise statement and complete proof.


%

\section{Central Model: Estimation Error Bounds}
\label{sec:central}

In this section, we turn to study label DP in the central model where the learning agent has access to the clear-text dataset $\cD$ and it only needs to guarantee the estimator is ``insensitive'' with respect to any single change of the label. Under this weaker privacy model, we show that the estimation error can be greatly improved compared to those in the local model. In the main paper, we will mainly focus on $\ell_2$-norm bounds and leave semi-norm bounds to Appendix~\ref{sec:central-semi}.
\subsection{Lower Bound}
We first have the following lower bound on estimation error, the proof of which is given in Appendix~\ref{proof:central-lower}.
\begin{theorem}
\label{thm:central-lower}
   For a large enough $n$, any estimator $\hat{\theta}$ based on samples form the BTL model that satisfies $(\epsilon,\delta)$-label DP in the central model has the estimation error in $\ell_2$-norm lower bounded as
   \begin{align*}
   \ex{\norm{\hat{\theta} - \theta^*}^2} \ge \Omega\left(
    \frac{d^2}{nL^2} +   \frac{d}{n^2(\epsilon + \delta)^2}\right).
\end{align*}
\end{theorem}
    Let us compare our lower bound with a similar one (although via a different approach) established in~\cite{cai2023score}, which enforces privacy protection for both label and features (i.e., standard DP notion rather than label DP). 
    If $L = O(\sqrt{d})$ (which holds for Gaussian design), then the first term is the same as in~\cite{cai2023score} and is equal to the standard non-private mean-square-error (MSE) lower bound \citep{hsu2023sample}. The main difference is the dependence on dimension in the second term, i.e., $d$ in our bound vs. $d^2$ in~\cite{cai2023score}. This improvement is due to the fact that our privacy protection is only for the scalar label, whereas \citet{cai2023score} also protects $d$-dimensional feature vectors (albeit in a different application than ours). 

\subsection{Algorithm and Upper Bound}
In this section, we present our algorithm and its privacy and estimation error guarantees.

Our algorithm builds upon the classic technique -- objective perturbation~\citep{kifer2012private}, i.e., it adds an additional noise term in the objective function. In particular, our estimator is given by 
\begin{align*}
   \hat{\theta}_{\texttt{obj}} = \argmin_{\theta \in \Theta_B}\,  {l}_{\cD}(\theta) + \frac{\beta}{2} \norm{\theta}_2^2 +  w^{\top} \theta~,
\end{align*}
where $l_{\cD}(\theta)$ is the negative log-likelihood defined in~\eqref{eq:mle}, $\beta >0$ is some regularizer and $w \sim \cN(0,\sigma^2 I)$ is an independent Gaussian noise.
We then have the following privacy guarantee. See Appendix~\ref{proof:central-priv} for the proof and  Algorithm~\ref{alg:objP} for pseudo-code. 
\begin{theorem}[Privacy]
\label{thm:central-priv}
     Let $\epsilon \!>\!0$, $\delta \!\in\! (0,1)$. Then, setting $\sigma = \frac{L \sqrt{8\log(2/\delta) + 4\epsilon}}{\epsilon}$ under Assumption~\ref{ass:bound},
Algorithm~\ref{alg:objP} satisfies $(\epsilon,\delta)$-label DP in the central model.
\end{theorem}

The next theorem provides the estimation error of $\hat{\theta}_{\texttt{obj}}$ in $\ell_2$-norm. See Appendix~\ref{proof:central-upper} for the full proof. 
\begin{theorem}[Estimation error]
\label{thm:central-upper}
    Let $\alpha \in (0,1)$. Then, under Assumptions~\ref{ass:bound} and~\ref{ass:cov}, with probability at least $1-\alpha$, $\hat{\theta}_{\texttt{obj}}$ satisfies
     \begin{align*}
\norm{\hat{\theta}_{\texttt{obj}} \!-\! \theta^*}_2 \!\le \!O\!\left(\!\frac{L}{\kappa \gamma}\sqrt{\frac{\log(1/\alpha)}{n}} \!+\! \frac{\sigma \!\big(\!\sqrt{d} \!+\! \sqrt{\log(1/\alpha)}\!\big)}{n \kappa \gamma}\!\right)
     \end{align*}
     where $\gamma:= \frac{1}{2 + \exp(-2LB) + \exp(2LB)}$ and $\kappa$ is the coverage coefficient in Assumption~\ref{ass:cov}.
\end{theorem}
\begin{corollary}\label{cor:central-upper}
For $\epsilon \in (0,1]$ and $\delta \in (0,1)$, let $\sigma = \frac{L \sqrt{8\log(2/\delta) + 4\epsilon}}{\epsilon}$ as in Theorem~\ref{thm:central-priv}, then the estimation error is of the order 
\begin{align*}
\norm{\hat{\theta}_{\texttt{obj}} - \theta^*}_2 \le \widetilde{O}\left( \frac{L}{\kappa \gamma \sqrt{n}} + \frac{L\sqrt{d\log(1/\delta)}}{n \epsilon \kappa \gamma}\right).
\end{align*}
\end{corollary}

\begin{remark}[Central vs. Local Models]
\label{rem:cvsl}
A key observation here is that the cost to ensure label DP is an additive lower-order term of the order $\frac{1}{\epsilon n}$ under the central model. In contrast, under the local model, the privacy cost is of the order $\frac{1}{(e^\epsilon-1)\sqrt{n}}$ (cf. Theorem~\ref{thm:sgd_rr}), which is approximately $\frac{1}{\epsilon\sqrt{n}}$ for high privacy regime (i.e., $\epsilon < 1$). This sharp decrease in privacy cost under the central model is due to the fact that it is a weaker privacy model; hence, instead of randomizing each label, one needs to add noise only once in the loss function.

\textbf{Comparison with Lower Bound.} We now compare the upper bound in Corollary~\ref{cor:central-upper} with lower bound in Theorem~\ref{thm:central-lower}.  Under Assumption~\ref{ass:bound} and in the best case when $\kappa = \Theta(L^2/d)$, we observe that the upper bound matches the lower bound up to a factor of $d e^{LB}$. If $L = B = \Theta(1)$, the gap between the bounds is on the order of $d$. Closing this gap and obtaining optimal error bounds is an open question.

\textbf{Extension to approximate minimizer.} Currently, the privacy guarantee in Theorem~\ref{thm:central-priv} only holds for the exact minimizer $\hat{\theta}_{\texttt{obj}}$. One can also add an additional output perturbation as in~\cite{bassily2019private,iyengar2019towards} to guarantee that an approximate minimizer (e.g., obtained by SGD) is private with the same order of estimation error. 

\end{remark}


\section{Simulations}

   \begin{figure*}[!t]
\begin{subfigure}[t]{.33\linewidth}
            \centering
			\includegraphics[width = 2in]{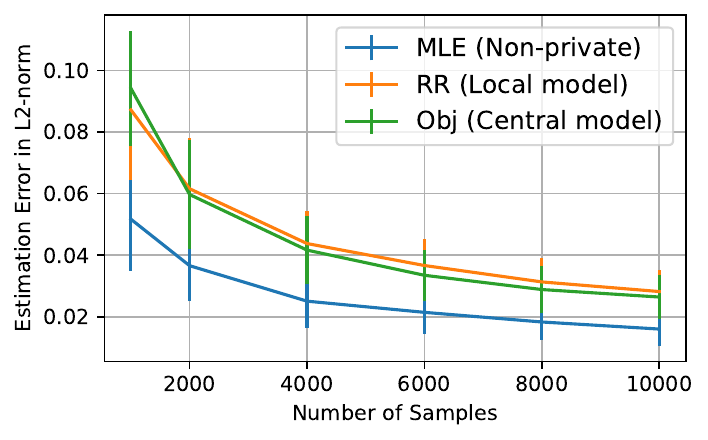}
   \vskip -1.5mm
			\caption{ $\epsilon=0.1$}
		\end{subfigure}\ \
  \begin{subfigure}[t]{.33\linewidth}
        \centering
		\includegraphics[width = 2in]{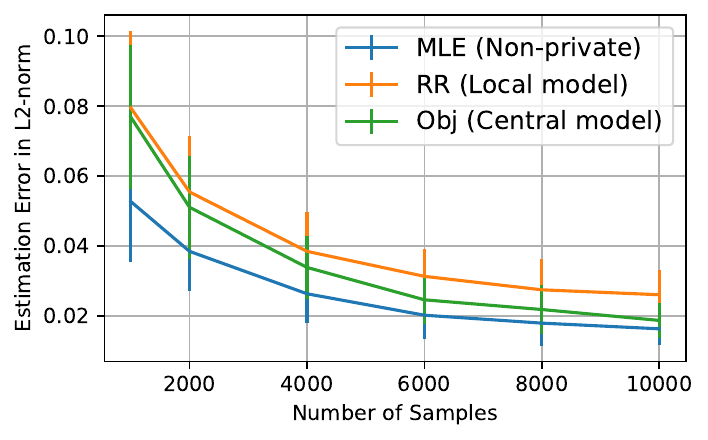}
  \vskip -1.5mm
			\caption{$\epsilon=0.5$ }
		\end{subfigure}\ \
  \begin{subfigure}[t]{.32\linewidth}
            \centering
			\includegraphics[width = 2in]{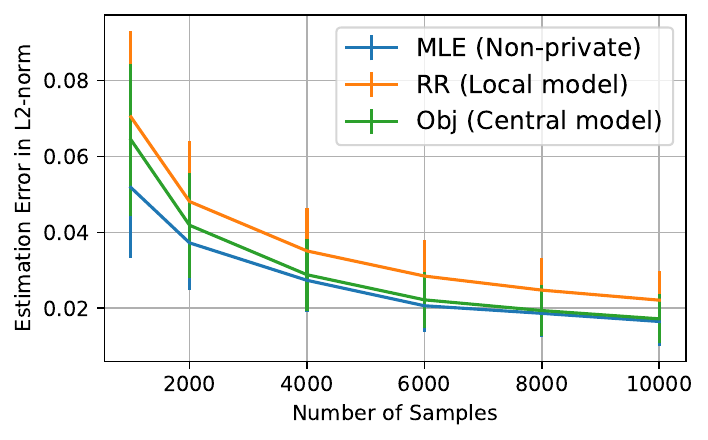}
   \vskip -1.5mm
			\caption{$\epsilon=1$}
		\end{subfigure}
		 \vspace{-2mm}
\caption{\footnotesize{Comparison of estimation error in $\ell_2$-norm between non-private MLE-based estimator, RR based estimator in local model and objective perturbation based estimator in central model for different privacy levels $\epsilon$.}  }\label{fig:L2_error}
		\vspace{-3mm}
\end{figure*}

We numerically evaluate the errors of our estimators under local and central models of label DP; and present comparisons with that of the non-private estimator of \cite{zhu2023principled}. Our simulations are proof-of-concept only; we do not tune any hyper-parameters.

We consider the BTL preference model of pairwise comparisons, with the number of samples $1000 \leq n \leq 10000$. Each sample size is repeated 100 times. We randomly generate $\theta^*$ from $d$-dimensional standard Gaussian, where we vary $d \in \lbrace 3,5,10\rbrace$. The state-action features $\phi$ are sampled iid, also from a $d$-dimensional standard Gaussian. The results for $d=5$ is shown in Figure~\ref{fig:L2_error}. 

For the local model, we use our SGD-based estimator $\hat \theta_{\texttt{SGD-RR}}$. For the central model, we implement our objective perturbation based estimator $\hat \theta_{\texttt{obj}}$ using SGD updates. To ensure consistency, we also implement the non-private estimator $\theta_{\texttt{MLE}}$ of \citet{zhu2023principled} using SGD updates. We use learning rate $\eta=0.1$ for all three estimators. We compare estimation error in $\ell_2$-norm for all the estimators for varying privacy levels $\epsilon \in \lbrace 0.1,0.5,1 \rbrace$. We fix $\delta=0.001$ for $\hat\theta_{\texttt{obj}}$.

We observe that estimation error decreases for all the estimators as the number of samples grows larger. 
Moreover, we observe that the non-private MLE-based estimator has the smallest estimation error, while the error in RR based estimator under local model is higher than the objective perturbation based estimator under the central model.
This is consistent with our theoretical results in
Sections \ref{sec:local} and \ref{sec:central}.

\section{Concluding Remarks}
We provided a systematic study of reward estimation via human feedback under the label DP. We also discuss the generalization to standard DP in Appendix~\ref{app:stdDP}.
For future directions, it is instructive to establish upper bounds with a dependency milder than $e^{LB}$. Along the lines of \citet{bach2010self-new}, where risk bounds are considered, it might be possible to leverage self-concordance properties of log-loss to obtain tighter estimation error bounds. Another important direction is to empirically evaluate the performance of a downstream policy trained using the estimated reward model along the lines of \citet{zhu2023principled} under different privacy notions.

\section{Acknowledgments}
XZ is supported in part by NSF CNS-2153220. XZ would like to thank Adam Smith, Daniel Kifer, and Abhradeep Guha Thakurta for discussions on the objective perturbation paper~\citep{kifer2012private}. XZ would also like to thank Abhradeep Guha Thakurta, Satyen Kale, and Karan Singh for discussions on~\cite{agarwal2023differentially}, which highlights some subtlety in proving privacy guarantee in objective perturbation. XZ would also like to thank Yichen Wang for the discussions on~\cite{cai2023score}.

\printbibliography

\appendix
\newpage 
\section{Additional Related Work}
\label{app:addRelated}

\paragraph{Preference-based Learning.}

\citet{shah2015estimation} study the problem of reward estimation under the pairwise (BTL and Thurstone) model and $K$-wise (PL) comparisons model. They work in the tabular setting and provide minimax error bounds for estimated rewards under both semi-norm and $\ell_2$-norm.~\citet{zhu2023principled} 
consider linearly parameterized rewards under the BTL and PL comparisons model, and prove the error bound of the maximum likelihood estimator. They further use this estimator to learn a pessimistic policy for an offline contextual bandit task and bound its sub-optimality gap. Multi-armed bandits (both tabular and parametric) under dueling/preference-based feedback are considered in a range of work, and different notions of regret guarantees are established; see~\citet{bengs2021preference} for a comprehensive survey.
\citet{pacchiano2021dueling} consider the episodic online RL problem under the BTL comparison model with tabular latent rewards and prove sublinear \emph{regret} guarantees in the number of episodes. \citet{chen2022human} generalize this to latent rewards with function approximation and establish sublinear regret bounds. \citet{zhan2023provable} consider the offline RL problem in the function approximation framework and prove corresponding performance guarantees. In contrast to all these prior work, we consider learning under the privacy of human labelers, where the pairwise comparisons provided by them is considered to be sensitive information. It is worth noting that~\citet{cai2023score} also consider privacy protection under the BTL model. Some key differences are as follows: (i) they study the tabular case rather than our linear reward model; (ii) their privacy notion is also different from ours in that it protects all the outcomes of a single item, rather than our label DP notion.

\paragraph{Label Differential Privacy.}

Privacy of pairwise comparisons can be protected using the notion of Label Differential Privacy. \citet{chaudhuri2011differentially} introduce this notion for the first time in learning theory and design private PAC learners under label DP. Since then, several follow-up work consider this notion of DP (see, e.g. \citet{beimel2013private,ghazi2021deep,esfandiari2022label}) where standard DP seems to be an overkill. The most related work to ours is \citet{ghazi2021deep}, which considers training deepNNs under label DP. Our work differs from theirs as well as from the literature on private stochastic optimization \cite{chaudhuri2008privacy,bassily2014private,kifer2012private,bassily2019private,song2021evading} in the utility guarantees (i.e. performance metrics); we consider bounding the error of parameter estimates under some metric, whereas these papers bound generalization errors or population risks.


\section{Additional Details on Sections~\ref{sec:lower}}
\label{app:lower}
We first introduce some necessary backgrounds and notations for our proofs on lower bounds, i.e., the lower model in this section and the central model in Appendix~\ref{app:central}.

Let $\cP$ be a family of distributions over $\cX^n$, where $\cX$ is the data universe and $n$ is the sample size. Let $\theta: \cP \to \Theta$ be the parameter of the distribution that we aim to estimate and let $\rho: \Theta \times \Theta \to \mathbb{R}_+$ be a pseudo-metric that is the loss function for estimating $\theta$. The minimax risk of estimation under loss $\rho$ for the class $\cP$ is 
\begin{align}
\label{eq:mini-np}
    R(\cP,\rho) := \min_{\hat{\theta}}\max_{P \in \cP} \mathbb{E}_{X \sim P}\left[\rho(\hat{\theta}(X), \theta(P))\right].
\end{align}
For $(\epsilon,\delta)$-label DP in the central model, the minimax risk is  
\begin{align*}
    R_c(\cP,\rho,\epsilon,\delta) := \min_{\hat{\theta} \text{ is} (\epsilon,\delta) \text{-label DP}  }\max_{P \in \cP} \mathbb{E}_{X \sim P}\left[\rho(\hat{\theta}(X), \theta(P))\right].
\end{align*}
For $(\epsilon,\delta)$-label DP in the local model, the minimax risk is  
\begin{align*}
    R_l(\cP,\rho,\epsilon,\delta) := \min_{Q \text{ is } (\epsilon,\delta) \text{-label DP mechanism} } \min_{\hat{\theta}}\max_{P \in \cP} \mathbb{E}_{X \sim P,Q}\left[\rho(\hat{\theta}(X), \theta(P))\right].
\end{align*}
For pure-DP where $\delta = 0$, we simply write $R_c(\cP,\rho,\epsilon)$ and $R_l(\cP,\rho,\epsilon)$. 

In this work, under BTL model, our goal is essentially to estimate the unknown parameter $\theta$ in the logistic model/distribution. In particular, we consider a fixed design (i.e., $x_i \in \Real^d$ is known) and the goal is to infer unknown $\theta$ after observing (private) sequence $y_i$, where the non-private $y_i$ is drawn from
\begin{align*}
\prob{y_i=1|x_i}= \sigma(\theta^\top x_i)=\frac{1}{1+\exp(-\theta^\top x_i)}~,\quad  \prob{y_i=0|x_i}= 1-\sigma(\theta^\top x_i).
\end{align*}
We denote this family of distribution by $\cP_{\log}$. More specifically, for the local model, the learner/agent aims to estimate $\theta$ from a sequence of private $\widetilde{y}_i$ generated by an LDP mechanism $Q$, while under the central model, the goal is to output an estimate $\hat{\theta}$ that is close to $\theta$ while guaranteeing label DP. Here, $\rho$ will be either squared $\ell_2$ norm or squared semi-norm. 

The following result will be useful in our proofs.
\begin{claim}
    \label{clm:kl}
       Let $p_a: = 1/(1+e^a)$ and $p_b: =  1/(1+e^b)$, we have $$\mathrm{kl}(p_a \| p_b) +  \mathrm{kl}(p_b \| p_a) \le (a-b)^2,$$ 
       where $\mathrm{kl}(p \| q): = \kl{\mathrm{Bernoulli}(p)}{\mathrm{Bernoulli}(q)}$ denotes KL-divergence between Bernoulli distributions with parameters $p$ and $q$.
\end{claim}
\begin{proof}
    By a direct calculation, we have 
    \begin{align*}
        \mathrm{kl}(p_a \| p_b) +  \mathrm{kl}(p_b \| p_a) = (p_a - p_b)\log\left(\frac{p_a}{1-p_a} \frac{1-p_b}{p_b}\right).
    \end{align*}
    Further, by the definition of $p_a$, $p_b$, we have 
    \begin{align*}
       (p_a - p_b)\log\left(\frac{p_a}{1-p_a} \frac{1-p_b}{p_b}\right) =\left( \frac{1}{1+e^a} - \frac{1}{1+e^b} \right) \cdot (b-a).
    \end{align*}
     Without loss of generality, we assume $b \ge a$, then 
     \begin{align*}
         \left( \frac{1}{1+e^a} - \frac{1}{1+e^b} \right) \le \frac{e^b - e^a}{e^b} = 1- e^{a-b} \le 1- (1+a-b) = b-a.
     \end{align*}
     Combining the above, yields the result.
\end{proof}

\subsection{Proof of Theorem~\ref{thm:semi-lower}}
\label{proof:semi-lower}
Before we state the main proof, let us first present some useful lemmas. In particular, as mentioned in the main paper, for the semi-norm part, we will rely on Fano's lemma to derive the minimax lower bound. The key idea is to construct a proper packing rather than restricting to hypercubes as in Assouad's lemma. Let us first recall the non-private Fano's lemma~\citep{yu1997assouad} as follows.
\begin{lemma}[Fano's Lemma]
    Let $\cV = \{P_1, P_2,\ldots, P_M\} \subseteq \cP$ such that for all $i \neq j$,  
    \begin{align*}
    \kl{P_i}{P_j} \le \beta~,\quad
        \rho'(\theta(P_i), \theta(P_j)) \ge \tau
    \end{align*}
    for a semi-metric $\rho'$ and some $\tau,\beta > 0$.
    Then, we have 
    \begin{align*}
        R(\cP, (\rho')^2) \ge \frac{\tau^2}{4} \left(1 - \frac{\beta + \log 2}{\log M}\right).
    \end{align*}
\end{lemma}
To construct a proper packing, Varshamov–Gilbert’s bound (cf.~\cite{guntuboyina2011lower}) will be useful. 
\begin{lemma}[Varshamov–Gilbert’s bound]
\label{lem:VG}
    For any $\xi \in (0,1/2)$ and for every dimension $d\ge 1$, there exist $M \ge e^{\frac{\xi^2 d}{2}}$ and $w_1, \ldots, w_M \in \{0,1\}^d$ such that 
    \begin{align*}
      d_{\mathrm{ham}}(w_i, w_j) \ge (1/2 - \xi) d, \quad  \forall i\neq j.
    \end{align*}
\end{lemma}
Now, we are ready to prove Theorem~\ref{thm:semi-lower}.

\begin{proof}[Proof of Theorem~\ref{thm:semi-lower}]
        We divide it into non-private and private parts. 

    \textbf{Non-private part.} Let $\xi = 1/4$ in Lemma~\ref{lem:VG}, then there exist $M \ge e^{\frac{\xi^2 d}{2}}$ and $w_1, \ldots, w_M \in \{0,1\}^d$ such that 
    \begin{align*}
        \forall i\neq j, \quad  \frac{d}{4} \le \norm{w_i - w_j}_2^2 \le  d.
    \end{align*}
    Now, let the eigenvalue decomposition of $\Sigma_{\cD} + \lambda I$ be $U^{\top} \Lambda U$ and $\theta_i := \frac{\Delta}{\sqrt{d}} U^{\top} \sqrt{\Lambda^{-1}} w_i$, then we have 
    \begin{align*}
        \norm{\theta_i - \theta_j}_{\Sigma_{\cD} + \lambda I}^2  &= (\theta_i - \theta_j)^{\top} (\Sigma_{\cD} + \lambda I) (\theta_i - \theta_j)\\
        &= \frac{\Delta^2}{d} (w_i - w_j)^{\top} \sqrt{\Lambda^{-1}} U (U^{\top}\Lambda U)U^{\top} \sqrt{\Lambda^{-1}} (w_i - w_j)\\
        &= \frac{\Delta^2}{d} \norm{w_i - w_j}_2^2.
    \end{align*}
    Thus, we have constructed a packing such that for any $i, j \in [M]$ and $i \neq j$, 
    \begin{align*}
        \Delta  \ge \norm{\theta_i - \theta_j}_{\Sigma_{\cD} + \lambda I} \ge \frac{\Delta}{2}.
    \end{align*}
    Now, let us turn to the KL divergence part. Let $P_i^n$ be the product distribution when $\theta^* = \theta_i$. Then, by chain rule of KL divergence and Claim~\ref{clm:kl}, we have 
    \begin{align}
    \label{eq:kl}
        \kl{P_i^n}{P_j^n} \le \sum_{k=1}^n (x_k^{\top} (\theta_i - \theta_j))^2  = n (\theta_i-\theta_j)^{\top} \Sigma_{\cD} (\theta_i - \theta_j) \le n \norm{\theta_i - \theta_j}_{\Sigma_{\cD} + \lambda I}^2\le  n\Delta^2.
    \end{align}
    Thus, by Fano's lemma, we have 
    \begin{align*}
R(\cP_{\log},\norm{\cdot}_{\Sigma_{\cD} + \lambda I}^2) \ge \frac{\Delta^2}{8}\left(1 - 32 \cdot \frac{n\Delta^2 + \log 2}{d}\right).
    \end{align*}
    Thus, choosing $\Delta^2 = c \frac{d}{n}$ for some constant $c$, we have for large $d$, 
     \begin{align*}
R(\cP_{\log},\norm{\cdot}_{\Sigma_{\cD} + \lambda I}^2) \ge \Omega\left(\frac{d}{n}\right).
    \end{align*}
     Finally, we also need to check that $\norm{\theta_i} \le B$ when $n$ is large. To this end, by the fact that $w_i \in \{0,1\}^d$ for any $i$, we have that
    \begin{align*}
        \norm{\theta_i} \le \frac{\Delta}{\sqrt{d}} \sqrt{\tr(\Lambda^{-1})} = \frac{\Delta}{\sqrt{d}} \sqrt{\tr( (\Sigma_{\cD} + \lambda I)^{-1})} \le B,
    \end{align*}
    where the last step holds when $n \ge \frac{c \tr( (\Sigma_{\cD} + \lambda I)^{-1})}{B^2}$, since $\Delta^2 = c\frac{d}{n}$. We also note that the centered condition  $\inner{1}{\theta} = 0$ can be simply achieved by reducing $d$ to $d/2$.

    \textbf{Private part.} Let $M_i^n$ be the product distribution of private view when $\theta^* = \theta_i$. The only change here is the KL divergence part. By Corollary 3 in~\citet{duchi2018minimax}, Pinsker's inequality and Claim~\ref{clm:kl}, we can obtain that 
    \begin{align*}
        \kl{M_i^n}{M_j^n} \le n (e^{\epsilon}-1)^2 (\theta_i-\theta_j)^{\top} \Sigma_{\cD} (\theta_i - \theta_j) \le n (e^{\epsilon}-1)^2 \Delta^2.
    \end{align*}
    Thus, a similar analysis as the non-private case gives 
    \begin{align*}
R(\cP_{\log},\norm{\cdot}_{\Sigma_{\cD} + \lambda I}^2,\epsilon) \ge \Omega\left(\frac{d}{n (e^{\epsilon}-1)^2}\right).
    \end{align*}
\end{proof}

\subsection{Proof of Theorem~\ref{thm:l2-lower}}
\label{proof:l2-lower}
Before we present the proof, we first introduce some useful lemmas. 
A convenient way to establish a lower bound in $\ell_2$ norm is via Assouad’s lemma. We restate it below with proof for completeness and some additional implications, which are useful for our proof.
\begin{lemma}[Assouad's lemma]
\label{lem:assouad}
    Let $\cV \subseteq \cP$ be a set of distributions indexed by the hypercube $\cE_d = \{\pm 1\}^d$. Suppose there exists a $\tau \in \Real$ and $\alpha >0$, such that $\rho$ satisfies: (i) for all $u, v, w \in \cE_d$, $\rho(\theta(P_u),\theta(P_v)) \ge 2\tau \cdot \sum_{i=1}^d \mathds{1}(u_i \neq v_i)$ and  (ii)  $\rho(\theta(P_u), \theta(P_v)) \le \alpha (\rho(\theta(P_u), \theta(P_w)) + \rho(\theta(P_v), \theta(P_w)) )$, i.e., $\alpha$-triangle inequality. For each $i\in [d]$, define the mixture distributions:
    \begin{align*}
        P_{+i}:= \frac{2}{|\cE_d|}\sum_{e\in \cE_d: e_i=1} P_e \text{ and } P_{+i}:= \frac{2}{|\cE_d|}\sum_{e\in \cE_d: e_i=-1} P_e.
    \end{align*}
    Then, we have 
    \begin{align*}
        R(\cP, \rho) \ge \frac{\tau}{2\alpha} \sum_{i = 1}^d \left(1 - \norm{P_{+i} - P_{-i}}_{\mathrm{TV}}\right).
    \end{align*}
\end{lemma}
\begin{proof}
    Let $P \in \cV \subseteq \cP$ and $X \sim P$. For any estimator $\hat{\theta}(X)$, define $\psi^* = \argmin_{e \in \cE_d} \rho(\hat{\theta}, \theta(P_e))$. Thus, we have 
    \begin{align*}
        \rho(\theta(P), \theta(P_{\psi^*})) \lep{a} \alpha \left(\rho(\theta(P), \hat{\theta}) +  \rho(\theta(P_{\psi^*}), \hat{\theta}) \right) \lep{b} 2\alpha \cdot \rho(\theta(P), \hat{\theta}),
    \end{align*}
    where (a) holds by the $\alpha$-triangle inequality of $\rho$; (b) holds by the definition of $\psi^*$. As a result, we have 
    \begin{align*}
        R(\cP, \rho)\ge R(\cV, \rho) = \min_{\hat{\theta}} \max_{P \in \cV} \mathbb{E}_{X \sim P}\left[\rho(\hat{\theta}(X), \theta(P))\right] \ge \frac{1}{2\alpha}\min_{\hat{\theta}} \max_{P \in \cV} \mathbb{E}_{X \sim P}\left[\rho({\theta}(P), \theta(P_{\psi^*}))\right].
    \end{align*}
    Now, by condition (i) of the loss $\rho$, we have 
    \begin{align*}
        \max_{P \in \cV} \mathbb{E}_{X \sim P}\left[\rho({\theta}(P), \theta(P_{\psi^*}))\right] &\gep{a} \frac{1}{|\cE_d|} \sum_{e \in \cE_d} \mathbb{E}_{X \sim P_e}\left[\rho({\theta}(P_e), \theta(P_{\psi^*}))\right] \\
        &\gep{b} \frac{2\tau}{|\cE_d|} \sum_{e \in \cE_d} \sum_{i=1}^d \mathbb{E}_{X \sim P_e}\left[\mathds{1}(\psi^*_i \neq e_i)\right]\\
        &\ep{c} \frac{2\tau}{|\cE_d|}  \sum_{e \in \cE_d} \sum_{i=1}^d \mathbb{P}_e \left[ \psi_i^* \neq e_i \right]\\
        &\ep{d} \frac{2\tau}{|\cE_d|} \sum_{i=1}^d \sum_{e \in \cE_d}  \mathbb{P}_e \left[ \psi_i^* \neq e_i \right],
    \end{align*}
    where (a) holds by maximum is larger than average; (b) holds by condition (i) of $\rho$; in (c), $\mathbb{P}_e$ is the probability measure when samples are generated from $P_e$; (d) follows by swapping the two sums.
    For each $i \in d$, we divide the set $\cE_d$ into two parts based on the value at $i$, i.e., 
    \begin{align*}
        \frac{2\tau}{|\cE_d|} \sum_{e \in \cE_d}  \mathbb{P}_e \left[ \psi_i^* \neq e_i \right] &= \frac{2\tau}{|\cE_d|} \sum_{e \in \cE_d: e_i = +1}  \mathbb{P}_e \left[ \psi_i^* \neq e_i \right] +  \frac{2\tau}{|\cE_d|} \sum_{e \in \cE_d: e_i = -1}  \mathbb{P}_e \left[ \psi_i^* \neq e_i \right]\\
        &= \tau \cdot \left(\mathbb{P}_{X\sim P_{+i}}\left[\psi_i^*(X) \neq 1\right] + \mathbb{P}_{X \sim P_{-i}}\left[\psi_i^*(X) \neq -1\right]\right).
    \end{align*}
    Combining the above together, we have 
    \begin{align*}
        R(P,\rho) &\ge \frac{\tau}{2\alpha} \sum_{i=1}^d \left(\mathbb{P}_{X\sim P_{+i}}\left[\psi_i^*(X) \neq 1\right] + \mathbb{P}_{X \sim P_{-i}}\left[\psi_i^*(X) \neq -1\right]\right)\\
        &\ge \frac{\tau}{2\alpha} \sum_{i=1}^d \inf_{\Psi}\left(\mathbb{P}_{X\sim P_{+i}}\left[\Psi(X) \neq 1\right] + \mathbb{P}_{X \sim P_{-i}}\left[\Psi(X) \neq -1\right]\right)\\
        &\ep{a} \frac{\tau}{2\alpha} \sum_{i = 1}^d \left(1 - \norm{P_{+i} - P_{-i}}_{\mathrm{TV}}\right),
    \end{align*}
    where (a) holds by Le Cam's first lemma. 
\end{proof}
 
\begin{corollary}
\label{cor:assouad}
    Under the same conditions of Lemma~\ref{lem:assouad}, we have 
        \begin{align*}
        R(\cP, \rho) &\ge \frac{d\tau}{2\alpha} \left[1 - \left(\frac{1}{d}\sum_{i=1}^d \frac{1}{2^d} \sum_{e \in \cE_d} \norm{P_{e} - P_{\bar{e}^i}}_{\mathrm{TV}}^2\right)^{1/2}\right],
    \end{align*}
    where $\bar{e}^i$ is a vector in $\cE_d$ that flips the $i$-th coordinate of $e$.
\end{corollary}
\begin{proof}
To start with, we introduce the following additional notations. For any $e \in \cE_d$, let $P_{e,+i}$ be the distribution indexed by first choosing $e$ and then letting $e_i = +1$. Similarly, we have $P_{e,-i}$. By this definition, we can rewrite $P_{+i}$ and $P_{-i}$ above as follows:
 \begin{align}
 \label{eq:rewrite}
     P_{+i} = \frac{1}{|\cE_d|} \sum_{e \in \cE_d} P_{e, +i} \text{ and } P_{-i} = \frac{1}{|\cE_d|} \sum_{e \in \cE_d} P_{e, -i}~.
 \end{align}
By Lemma~\ref{lem:assouad}, we have 
   \begin{align}
   \label{eq:tv}
        R(\cP, \rho) \ge \frac{\tau}{2\alpha} \sum_{i = 1}^d \left(1 - \norm{P_{+i} - P_{-i}}_{\mathrm{TV}}\right).
    \end{align}
    Now note that
    \begin{align*}
        \sum_{i=1}^d \norm{P_{+i} - P_{-i}}_{\mathrm{TV}} &\lep{a} \sqrt{d} \left(\sum_{i=1}^d \norm{P_{+i} - P_{-i}}_{\mathrm{TV}}^2\right)^{1/2}\\
        &\lep{b} \sqrt{d} \left(\sum_{i=1}^d \frac{1}{|\cE_d|} \sum_{e \in \cE_d} \norm{P_{e,+i} - P_{e,-i}}_{\mathrm{TV}}^2\right)^{1/2}\\
        & = \sqrt{d} \left(\sum_{i=1}^d \frac{1}{2^d} \sum_{e \in \cE_d} \norm{P_{e,+i} - P_{e,-i}}_{\mathrm{TV}}^2\right)^{1/2}.
    \end{align*}
    where (a) holds by  Cauchy-Schwarz inequality; (b) holds by~\eqref{eq:rewrite} and joint convexity of $\norm{\cdot}_{\mathrm{TV}}^2$. Plugging it back to~\eqref{eq:tv} and rearranging, we have 
    \begin{align*}
        R(\cP, \rho) \ge \frac{\tau d}{2\alpha} \left[1 - \left(\frac{1}{d}
        \sum_{i=1}^d \frac{1}{2^d} \sum_{e \in \cE_d} \norm{P_{e,+i} - P_{e,-i}}_{\mathrm{TV}}^2 \right)^{1/2} \right],
    \end{align*}
    which finishes the first part. The final result in the corollary simply follows from that TV distance is symmetric.
\end{proof}

Now, we are ready to prove Theorem~\ref{thm:l2-lower}.
\begin{proof}[Proof of Theorem~\ref{thm:l2-lower}]
We also divide it into two parts: the non-private part and the private part.

\textbf{Non-private part.}     Choose some $\Delta >0$ and for each $e \in \cE_d = \{\pm 1\}^d$,  let $\theta_e = \Delta e$. Now we need to check the two conditions in Lemma~\ref{lem:assouad}. First note that $\rho = \norm{\cdot}_2^2$ satisfies $2$-triangle inequality, i.e., $\alpha = 2$. Also, note that for any $u, v \in \cE_d$, $\norm{\theta_u - \theta_v}_2^2 = 4\Delta^2 \sum_{i=1}^d \mathds{1}(u_i \neq v_i)$, i.e., $\tau = 2\Delta^2$. Thus, let $P_e^n$ be the distribution for the $n$ independent samples of  (non-private) $y_i$ when $\theta = \theta_e$ and then by Corollary~\ref{cor:assouad}, we have 
    \begin{align*}
R_l(\cP_{\text{log}},\norm{\cdot}_2^2,\epsilon,\delta) \ge R(\cP_{\text{log}},\norm{\cdot}_2^2) \ge \frac{d\Delta^2}{2} \left[1 - \left(\frac{1}{d}\sum_{i=1}^d \frac{1}{2^d} \sum_{e \in \cE_d} \norm{P_{e}^n - P_{\bar{e}^i}^n}_{\mathrm{TV}}^2\right)^{1/2}\right].
    \end{align*}
    Thus, it remains to bound the part of TV distance. By Pinsker's inequality and chain rule of KL-divergence, we have for any $u, v \in \cE_d$
    \begin{align*}
        \norm{P_{u}^n - P_{v}^n}_{\mathrm{TV}}^2 &\le \frac{1}{4}\left(\kl{P^n_u}{P^n_{v}} + \kl{P^n_{v}}{P^n_u}\right) \\
        &= \frac{1}{4}\sum_{k=1}^n \left(\mathrm{kl}(p_u(x_k) \| p_v(x_k)) + \mathrm{kl}(p_v(x_k) \| p_u(x_k))\right),
    \end{align*}
    Then, by Claim~\ref{clm:kl}, we can bound the TV-distance term as follows.
    \begin{align*}
         \norm{P_{u}^n - P_{v}^n}_{\mathrm{TV}}^2 \le \frac{\Delta^2}{4}\sum_{k=1}^n \left(x_k^{\top} (u-v)\right)^2.
    \end{align*}
    This directly implies that 
    \begin{align*}
       \frac{1}{d 2^d}\sum_{i=1}^d\sum_{e \in \cE_d} \norm{P_{e}^n - P_{\bar{e}^i}^n}_{\mathrm{TV}}^2 \le \frac{\Delta^2}{4d} \frac{1}{2^d}\sum_{e \in \cE_d} \sum_{i=1}^d \sum_{k=1}^n \left(2 x_{ki}\right)^2 =  \frac{\Delta^2}{d} \frac{1}{2^d}\sum_{e \in \cE_d} \sum_{i=1}^d \sum_{k=1}^n x_{ki}^2 =  \frac{\Delta^2}{d} \frac{1}{2^d}\sum_{e \in \cE_d} \norm{X}_{\mathrm{F}}^2,
    \end{align*}
    where $X \in \Real^{n \times d}$ and $x_k^{\top} \in \Real^d$ is the $k$-th row and $\norm{\cdot}_{\mathrm{F}}$ is the Frobenius norm. Hence, we obtain that 
    \begin{align*}
R_l(\cP_{\text{log}},\norm{\cdot}_2^2,\epsilon,\delta) \ge \frac{d\Delta^2}{2}\left[1 - \left(\frac{\Delta^2}{d} \norm{X}_{\mathrm{F}}^2\right)^{1/2}\right].
    \end{align*}
    Finally, choosing $\Delta^2 = \frac{d}{4 \norm{X}_{\mathrm{F}}^2}$, we have 
    \begin{align*}      R_l(\cP_{\text{log}},\norm{\cdot}_2^2,\epsilon,\delta) \ge \frac{d^2}{16 \norm{X}_{\mathrm{F}}^2} = \frac{d}{n}\cdot \frac{1}{16\frac{1}{dn} \sum_{k=1}^n \norm{x_k}^2}.
    \end{align*}
    Since $\norm{x_k}^2 \le L^2$, one can further simplify it as 
     \begin{align*}      R_l(\cP_{\text{log}},\norm{\cdot}_2^2,\epsilon,\delta) \ge \Omega\left(\frac{d}{L^2} \cdot \frac{d}{n}\right).
    \end{align*}
    Finally, note that one can indeed easily check that for large enough $n$, $\norm{\theta} \le B$ and also $\inner{1}{\theta} = 0$ by halving the dimension $d$.

    \textbf{Private part.}
Now, let us turn to the private part. In particular, let $M_e^n$ be the distribution for the $n$ independent samples of private view $\widetilde{y}_i$ when $\theta = \theta_e$ and then by Corollary~\ref{cor:assouad}, we have 
  \begin{align*}
R_l(\cP_{\text{log}},\norm{\cdot}_2^2,\epsilon,\delta)  \ge \frac{d\Delta^2}{2} \left[1 - \left(\frac{1}{d}\sum_{i=1}^d \frac{1}{2^d} \sum_{e \in \cE_d} \norm{M_{e}^n - M_{\bar{e}^i}^n}_{\mathrm{TV}}^2\right)^{1/2}\right].
    \end{align*}
Again, the key is to bound the TV-distance term. To this end, by Corollary 3 in~\citet{duchi2018minimax} and Pinsker's inequality, we have 
\begin{align*}
     \norm{M_{u}^n - M_{v}^n}_{\mathrm{TV}}^2 &\le \frac{(e^{\epsilon}-1)^2}{2}\sum_{k=1}^n \left(\mathrm{kl}(p_u(x_k) \| p_v(x_k)) + \mathrm{kl}(p_v(x_k) \| p_u(x_k))\right).
\end{align*}
Then, following the same analysis as the non-private case, we can obtain that 
\begin{align*}
R_l(\cP_{\text{log}},\norm{\cdot}_2^2,\epsilon,\delta) \ge \frac{d\Delta^2}{2}\left[1 - \left(\frac{2 (e^{\epsilon}-1)^2 \Delta^2}{d} \norm{X}_{\mathrm{F}}^2\right)^{1/2}\right].
    \end{align*}
    Finally, choosing $\Delta^2 = \frac{d}{8 (e^{\epsilon}-1)^2 \norm{X}_{\mathrm{F}}^2}$, we have 
    \begin{align*}      R_l(\cP_{\text{log}},\norm{\cdot}_2^2,\epsilon,\delta) \ge \frac{d^2}{32 (e^{\epsilon}-1)^2 \norm{X}_{\mathrm{F}}^2} = \frac{d}{n (e^{\epsilon}-1)^2}\cdot \frac{1}{32\frac{1}{dn} \sum_{k=1}^n \norm{x_k}^2}.
    \end{align*}
    Again, noting that $\norm{x_k}^2 \le L^2$, one can simplify it as
    \begin{align*}
R_l(\cP_{\text{log}},\norm{\cdot}_2^2,\epsilon,\delta) \ge \Omega\left(\frac{d}{L^2} \cdot \frac{d}{n (e^{\epsilon}-1)^2} \right).
    \end{align*}
\end{proof}

\section{Additional Details on Section~\ref{sec:local}}
\label{app:local}

We are given a query-observation dataset $\cD=(s_i,a^0_i,a^1_i,y_i)_{i=1}^{n}$. 
Define $x_i = \phi(s_i,a^1_i)-\phi(s_i,a^0_i)$. Under the BTL model, the labels $y_i \in \lbrace 0,1\rbrace$ be drawn from the distribution
\begin{align*}
p_{i,1}=\prob{y_i=1|x_i}= \sigma(\theta^\top x_i)=\frac{1}{1+\exp(-\theta^\top x_i)}~,\quad  p_{i,0}=\prob{y_i=0|x_i}= 1-\sigma(\theta^\top x_i)~.
\end{align*}
When queried the value of $y_i$, the RR mechanism outputs $\widetilde y_i$ with probability
\begin{align*}
    \prob{\widetilde y_i=y_i } = \sigma(\epsilon) = \frac{1}{1+\exp(-\epsilon)}~,\quad \prob{\widetilde y_i\neq y_i } = 1-\sigma(\epsilon)~.
\end{align*}
Then, for any $\theta \!\in\! \Real^d$, the predicted probabilities of a randomized label $\widetilde y_i$ given $x_i$ are
\begin{align*}
\widetilde p_{i,1}&= \prob{\widetilde y_i=1|x_i}= \sigma(\theta^\top x_i)\sigma(\epsilon) + (1-\sigma(\theta^\top x_i)) (1-\sigma(\epsilon))~,\\
 \widetilde p_{i,0}&= \prob{\widetilde y_i=0|x_i}=(1-\sigma(\theta^\top x_i))\sigma(\epsilon) + \sigma(\theta^\top x_i) (1-\sigma(\epsilon))~.
\end{align*}

\subsection{Proof of Theorem~\ref{thm:mle_rr}}
\label{proof:mle-rr}
First, the predicted probabilities for any $\theta \in \Real^d$ can be computed as
\begin{align*}
    \prob{\widetilde y_i=1|x_i} &= \frac{1}{1+\exp(-\inner{\theta}{x_i})} \cdot \frac{\exp(\epsilon)}{1+\exp(\epsilon)} + \frac{\exp(-\inner{\theta}{x_i})}{1+\exp(-\inner{\theta}{x_i})} \cdot \frac{1}{1+\exp(\epsilon)}
     = \frac{1+e^{-\epsilon} e^{-\theta^\top x_i}}{(1+e^{-\theta^\top x_i})(1+e^{-\epsilon}) }\\
    \prob{\widetilde y_i=0|x_i} &= \frac{\exp(-\inner{\theta}{x_i})}{1+\exp(-\inner{\theta}{x_i})} \cdot \frac{\exp(\epsilon)}{1+\exp(\epsilon)} + \frac{1}{1+\exp(-\inner{\theta}{x_i})} \cdot \frac{1}{1+\exp(\epsilon)}=\frac{e^{-\epsilon}+ e^{-\theta^\top x_i}}{(1+e^{-\theta^\top x_i})(1+e^{-\epsilon}) }.
\end{align*}
Based on this, the negative log-likelihood in \eqref{eq:mle_rr} takes the form
\begin{align*}
    \widetilde l_{\cD,\epsilon}(\theta) = -\frac{1}{n}\sum_{i=1}^n \left[ \mathds{1}(\widetilde y_i=1)\log \frac{1+e^{-\epsilon} e^{-\theta^\top x_i}}{(1+e^{-\theta^\top x_i})(1+e^{-\epsilon}) }+ \mathds{1}(\widetilde y_i=0)\log \frac{e^{-\epsilon}+ e^{-\theta^\top x_i}}{(1+e^{-\theta^\top x_i})(1+e^{-\epsilon}) }\right]~.
\end{align*}
Now the gradient of negative log-likelihood is given by $\nabla l_{\cD,\epsilon}(\theta) = -\frac{1}{n}\sum_{i=1}^n V_{\theta,i}x_i = -\frac{1}{n} X^\top V_\theta$, where
\begin{align*}
  V_{\theta,i}= \mathds{1}(\widetilde y_i=1)\left(\frac{e^{-\theta^\top x_i}}{1+e^{-\theta^\top x_i}} -\frac{e^{-\epsilon}e^{-\theta^\top x_i}}{1+e^{-\epsilon}e^{-\theta^\top x_i}} \right)  +   \mathds{1}(\widetilde y_i=0) \left(\frac{e^{-\theta^\top x_i}}{1+e^{-\theta^\top x_i}} - \frac{e^{-\theta^\top x_i}}{e^{-\epsilon}+e^{-\theta^\top x_i}}\right)~.
\end{align*}
It holds that 
\begin{align*}
    \mathbb{E}_\theta[V_{\theta,i}|x_i] &= \frac{e^{-\theta^\top x_i}}{1+e^{-\theta^\top x_i}} - \left( \frac{e^{-\epsilon}e^{-\theta^\top x_i}}{1+e^{-\epsilon}e^{-\theta^\top x_i}} \cdot \frac{1+e^{-\epsilon} e^{-\theta^\top x_i}}{(1+e^{-\theta^\top x_i})(1+e^{-\epsilon}) } + \frac{e^{-\theta^\top x_i}}{e^{-\epsilon}+e^{-\theta^\top x_i}} \cdot \frac{e^{-\epsilon}+ e^{-\theta^\top x_i}}{(1+e^{-\theta^\top x_i})(1+e^{-\epsilon}) } \right)\\
    & = \frac{e^{-\theta^\top x_i}}{1+e^{-\theta^\top x_i}} - \frac{e^{-\theta^\top x_i}}{1+e^{-\theta^\top x_i}} =0
\end{align*}
Now, under Assumption~\ref{ass:bound}, we have $-c \leq \theta^\top x_i \leq c$, where $c=O(LB)$. Hence, we have 
\begin{align*}
    V_{\theta,i} | (\widetilde y_i =1) &= \frac{e^{-\theta^\top x_i}(e^{\epsilon}-1)}{(1+e^{-\theta^\top x_i})(e^{\epsilon}+e^{-\theta^\top x_i})} \le \frac{e^\epsilon-1}{(e^\epsilon+e^{-c})} = \frac{e^{c}(e^\epsilon-1)}{(e^\epsilon e^c +1)}~,
    \\
    V_{\theta,i} | (\widetilde y_i =0) & = \frac{e^{-\theta^\top x_i}(e^{\epsilon}-1)}{(1+e^{-\theta^\top x_i})(1+e^{\epsilon}e^{-\theta^\top x_i})}
     \le \frac{(e^{\epsilon}-1)}{(e^{\epsilon}+e^{-c})} = \frac{e^{c}(e^\epsilon-1)}{(e^\epsilon e^c +1)}~.
\end{align*}
Therefore, it holds that $V_{\theta,i}$ is zero-mean and $v=\frac{e^{c}(e^\epsilon-1)}{(e^\epsilon e^c +1)}$-sub-Gaussian under the conditional distribution $\mathbb{P}_\theta[\cdot|x_i]$ and under Assumption~\ref{ass:bound}.

Now the Hessian of negative log-likelihood is given by $\nabla^2 l_{\cD,\epsilon}(\theta) = \frac{1}{n}\sum_{i=1}^n \left[ \mathds{1}(\widetilde y_i=1) \alpha_{1,i} + \mathds{1}(\widetilde y_i=0)\alpha_{0,i}\right]x_ix_i^\top$, where
\begin{align*}
    \alpha_{1,i} & = \frac{e^{-\theta^\top x_i}}{(1+e^{-\theta^\top x_i})^2} -\frac{e^{-\epsilon}e^{-\theta^\top x_i}}{(1+e^{-\epsilon}e^{-\theta^\top x_i})^2}
    = \frac{e^{-\theta^\top x_i}}{(1+e^{-\theta^\top x_i})^2}\cdot \frac{(e^\epsilon-1)(e^\epsilon e^{2\theta^\top x_i}-1)}{(1+e^\epsilon e^{\theta^\top x_i})^2}~,  \\
    \alpha_{0,i} & = \frac{e^{-\theta^\top x_i}}{(1+e^{-\theta^\top x_i})^2} - \frac{e^{-\theta^\top x_i}}{(e^{-\epsilon}+e^{-\theta^\top x_i})^2} = \frac{e^{-\theta^\top x_i}}{(1+e^{-\theta^\top x_i})^2}\cdot \frac{(e^\epsilon-1)(e^\epsilon e^{-2\theta^\top x_i}-1)}{(1+e^\epsilon e^{-\theta^\top x_i})^2}~.
\end{align*}
Under Assumption~\ref{ass:bound}, both $\alpha_{1,i},\alpha_{0,i} \ge \widetilde\gamma_\epsilon $ for all $\theta \in \Theta_B$, where
\begin{align*}
     \widetilde\gamma_\epsilon = \gamma\frac{(e^\epsilon-1)(e^\epsilon e^{-2c}-1)}{(e^\epsilon e^c+1)^2}~.
\end{align*}
Now $\widetilde\gamma_\epsilon > 0$ only when $\epsilon > 2c$.
This implies that $\widetilde l_{\cD,\epsilon}$ is $\widetilde\gamma_\epsilon$ strongly convex in $\theta_B$ when $\epsilon > 2c$ with respect to the semi-norm $\norm{\cdot}_{\Sigma_\cD}$. Since $\theta^* \in \Theta_B$, introducing the error vector $\Delta = \widetilde \theta_{\texttt{MLE}} -\theta^*$, we conclude that
\begin{align*}
\widetilde\gamma_\epsilon\norm{\Delta}^2_{\Sigma_\cD} \leq \norm{\nabla \widetilde l_{\cD,\epsilon}(\theta^*)}_{(\Sigma_\cD+\lambda I)^{-1}} \norm{\Delta}_{(\Sigma_\cD+\lambda I)}
\end{align*}
for some $\lambda > 0$.
Introducing $M = \frac{1}{n^2}X(\Sigma_\cD+\lambda I)^{-1}X^\top$, we have $\norm{\nabla \widetilde l_{\cD,\epsilon}(\theta^*)}^2_{(\Sigma_\cD+\lambda I)^{-1}}=V_{\theta^*}^\top MV_{\theta^*}$. Then, the Bernstein's inequality for sub-Gaussian random variables in quadratic form
(see e.g. Theorem 2.1 in \citet{hsu2012tail}) implies that with probability at least $1-\alpha$,
\begin{align*}
   \norm{\nabla \widetilde l_{\cD,\epsilon}(\theta^*)}^2_{(\Sigma_\cD+\lambda I)^{-1}} =V_{\theta^*}^\top MV_{\theta^*} &\leq v^2\left(\tr(M) + 2 \sqrt{\tr(M^\top M)\log(1/\alpha)}+ 2\norm{M}\log(1/\delta) \right)\\
   & \leq C_1\cdot v^2\cdot\frac{d+\log(1/\alpha)}{n} ~
\end{align*}
for some constant $C_1 > 0$.
This gives us
\begin{align*}
\widetilde\gamma_\epsilon\norm{\Delta}^2_{\Sigma_\cD+\lambda I} &\leq \norm{\nabla \widetilde l_{\cD,\epsilon}(\theta^*)}_{(\Sigma_\cD+\lambda I)^{-1}} \norm{\Delta}_{(\Sigma_\cD+\lambda I)}  + 4 \lambda \widetilde\gamma_\epsilon B^2\\
& \leq \sqrt{C_1\cdot v^2\cdot\frac{d+\log(1/\alpha)}{n} }\norm{\Delta}_{(\Sigma_\cD+\lambda I)}  + 4 \lambda \widetilde\gamma_\epsilon B^2~.
\end{align*}
Solving for the above inequality, we get 
\begin{align*}
\norm{\Delta}_{(\Sigma_\cD+\lambda I)}  \leq C_2\cdot \sqrt{\frac{v^2}{\widetilde\gamma_\epsilon^2}\cdot\frac{d+\log(1/\alpha)}{n} +\lambda B^2 }~
\end{align*}
for some constant $C_2 > 0$.
Now note that $\frac{v}{\widetilde\gamma_\epsilon}=\frac{e^{c}(e^\epsilon e^c+1)}{\gamma(e^\epsilon e^{-2c}-1)} \le C_3\cdot\frac{(e^\epsilon e^{2c}+1)}{\gamma(e^\epsilon e^{-2c}-1)}$ for some constant $C_3 >0$. Substituting this, we get
\begin{align*}
  \norm{\widetilde\theta_{\texttt{MLE}}-\theta^*}_{(\Sigma_\cD+\lambda I)}  \leq C\cdot  \frac{(e^\epsilon e^{2c}+1)}{\gamma(e^\epsilon e^{-2c}-1)}\sqrt{\frac{d+\log(1/\alpha)}{n}}+ C'\cdot \sqrt{\lambda} B
\end{align*}
for some constants $C,C'>0$,
which holds for any $\epsilon > 2c$. Setting $c=O(LB)$ completes the proof.

\subsection{Proof of Theorem~\ref{thm:robust_rr}}
\label{proof:robust-rr}

First recall from \eqref{eq:robust_rr} our de-biased loss function
\begin{align*}
   \widehat l_{\cD,\epsilon}(\theta)&=
  -\frac{1}{n}\sum_{i=1}^n \Bigg[ \mathds{1}(\widetilde y_i=1)\left(\sigma(\epsilon)\log  \sigma(\theta^\top x_i) -(1-\sigma(\epsilon))\log (1-\sigma(\theta^\top x_i)) \right)\\& \quad\quad\quad\quad +\mathds{1}(\widetilde y_i=0)\left(\sigma(\epsilon) \log (1-\sigma(\theta^\top x_i)) - (1-\sigma(\epsilon))\log \sigma(\theta^\top x_i)\right)\Bigg]~.
\end{align*}  
The gradient of the loss function is given by $\nabla \widehat l_{\cD,\epsilon}(\theta) = -\frac{1}{n}\sum_{i=1}^n V_{\theta,i}x_i = -\frac{1}{n} X^\top V_\theta$, where
\begin{align*}
  V_{\theta,i}= \mathds{1}(\widetilde y_i=1)\left( \frac{\sigma'(\theta^\top x_i)}{\sigma(\theta^\top x_i)}\sigma(\epsilon)+\frac{\sigma'(\theta^\top x_i)}{1-\sigma(\theta^\top x_i)}(1-\sigma(\epsilon))\right)  -   \mathds{1}(\widetilde y_i=0) \left( \frac{\sigma'(\theta^\top x_i)}{1-\sigma(\theta^\top x_i)}\sigma(\epsilon)+\frac{\sigma'(\theta^\top x_i)}{\sigma(\theta^\top x_i)}(1-\sigma(\epsilon))\right)~.
\end{align*}
It holds that 
\begin{align*}
\mathbb{E}_\theta[V_{\theta,i}|x_i] &= \left( \sigma(\theta^\top x_i)\sigma(\epsilon) + (1-\sigma(\theta^\top x_i)) (1-\sigma(\epsilon))\right) \left( \frac{\sigma'(\theta^\top x_i)}{\sigma(\theta^\top x_i)}\sigma(\epsilon)+\frac{\sigma'(\theta^\top x_i)}{1-\sigma(\theta^\top x_i)}(1-\sigma(\epsilon))\right)\\  
&- \left( (1-\sigma(\theta^\top x_i))\sigma(\epsilon) + \sigma(\theta^\top x_i) (1-\sigma(\epsilon)) \right)   \left( \frac{\sigma'(\theta^\top x_i)}{1-\sigma(\theta^\top x_i)}\sigma(\epsilon)+\frac{\sigma'(\theta^\top x_i)}{\sigma(\theta^\top x_i)}(1-\sigma(\epsilon))\right)\\
& = 0~.
\end{align*}
Furthermore, we have 
\begin{align*}
   \abs{ V_{\theta,i}}_{\widetilde y_i =1} &= \frac{\sigma'(\theta^\top x_i)}{\sigma(\theta^\top x_i)}\sigma(\epsilon)+\frac{\sigma'(\theta^\top x_i)}{1-\sigma(\theta^\top x_i)}(1-\sigma(\epsilon))~,\\
   \abs{ V_{\theta,i}}_{\widetilde y_i =0} &= \frac{\sigma'(\theta^\top x_i)}{1-\sigma(\theta^\top x_i)}\sigma(\epsilon)+\frac{\sigma'(\theta^\top x_i)}{\sigma(\theta^\top x_i)}(1-\sigma(\epsilon))~.
\end{align*}
The first derivative of the logistic function $\sigma(\cdot)$ is given by $\sigma'(z) = \sigma(z)(1-\sigma(z))$, which gives us
\begin{align*}
   \abs{ V_{\theta,i}}_{\widetilde y_i =1} &= (1-\sigma(\theta^\top x_i))\sigma(\epsilon)+\sigma(\theta^\top x_i)(1-\sigma(\epsilon))=\mathbb{P}_\theta[\widetilde y_i=0|x_i] \\  \abs{ V_{\theta,i}}_{\widetilde y_i =0} &= \sigma(\theta^\top x_i)\sigma(\epsilon)+(1-\sigma(\theta^\top x_i))(1-\sigma(\epsilon)) = \mathbb{P}_\theta[\widetilde y_i=1|x_i]~.
\end{align*}
Therefore, it holds that $V_{\theta,i}$ is zero-mean and $v=1$ sub-Gaussian under the conditional distribution $\mathbb{P}_\theta[\cdot|x_i]$~.

Now the Hessian of the loss function is given by 
\begin{align*}
 \nabla^2 \widehat l_{\cD,\epsilon}(\theta) &= \frac{1}{n}\sum_{i=1}^n \Bigg[ \mathds{1}(\widetilde y_i=1) \left((1-\sigma(\epsilon))\nabla^2\log (1-\sigma(\theta^\top x_i))-\sigma(\epsilon)\nabla^2 \log  \sigma(\theta^\top x_i)\right)\\
 &\quad\quad\quad\quad+ \mathds{1}(\widetilde y_i=0)\left((1-\sigma(\epsilon))\nabla^2\log \sigma(\theta^\top x_i)-\sigma(\epsilon)\nabla^2 \log  (1-\sigma(\theta^\top x_i))\right)\Bigg]~,
\end{align*}
where
\begin{align*}
  \nabla^2 \log  \sigma(\theta^\top x_i) & = \frac{\sigma''(\theta^\top x_i)\sigma(\theta^\top x_i)-\sigma'(\theta^\top x_i)^2}{\sigma(\theta^\top x_i)^2}x_ix_i^\top,\\ \nabla^2 \log (1- \sigma(\theta^\top x_i)) & = -\frac{\sigma''(\theta^\top x_i)(1-\sigma(\theta^\top x_i))+\sigma'(\theta^\top x_i)^2}{(1-\sigma(\theta^\top x_i))^2}x_ix_i^\top~.
\end{align*} 
Now the second derivative of the logistic function $\sigma(\cdot)$ is given by
$\sigma''(z) = \sigma'(z)(1-2\sigma(z))$, which gives us
\begin{align*}
 \nabla^2 \log  \sigma(\theta^\top x_i)  =   \nabla^2 \log (1- \sigma(\theta^\top x_i))  = - \sigma'(\theta^\top x_i)x_ix_i^\top~. 
\end{align*}
Hence, the Hessian of the loss function takes the form 
\begin{align*}
 \nabla^2 \widehat l_{\cD,\epsilon}(\theta) = \frac{1}{n}\sum_{i=1}^n \left[ \mathds{1}(\widetilde y_i=1)(2\sigma(\epsilon)-1)\sigma'(\theta^\top x_i) + \mathds{1}(\widetilde y_i=0)(2\sigma(\epsilon)-1)\sigma'(\theta^\top x_i)       \right]x_ix_i^\top~.
 \end{align*}
Now, under Assumption~\ref{ass:bound}, observe that $\sigma'(\theta^\top x_i) \ge \gamma$ for all $\theta \in \Theta_B$, where $\gamma = \frac{1}{2 + \exp(-2LB) + \exp(2LB)}$. This implies that $\widehat l_{\cD,\epsilon}$ is $\gamma_\epsilon:= \gamma (2\sigma(\epsilon)-1)$ strongly convex in $\Theta_B$ for all $\epsilon > 0$ with respect to the semi-norm $\norm{\cdot}_{\Sigma_\cD}$. Since $\theta^* \in \Theta_B$, introducing the error vector $\Delta = \hat \theta_{\texttt{RR}} -\theta^*$, we conclude that
\begin{align*}
\gamma_\epsilon\norm{\Delta}^2_{\Sigma_\cD} \leq \norm{\nabla \widehat l_{\cD,\epsilon}(\theta^*)}_{(\Sigma_\cD+\lambda I)^{-1}} \norm{\Delta}_{(\Sigma_\cD+\lambda I)}~
\end{align*}
for some $\lambda > 0$.
Introducing $M = \frac{1}{n^2}X(\Sigma_\cD+\lambda I)^{-1}X^\top$, we now have $\norm{\nabla l_{\cD,\epsilon}(\theta^*)}^2_{(\Sigma_\cD+\lambda I)^{-1}}=V_{\theta^*}^\top MV_{\theta^*}$. Then, the Bernstein's inequality for sub-Gaussian random variables in quadratic form
(see e.g. Theorem 2.1 in \citet{hsu2012tail}) implies that with probability at least $1-\alpha$,
\begin{align*}
   \norm{\nabla \widehat l_{\cD,\epsilon}(\theta^*)}^2_{(\Sigma_\cD+\lambda I)^{-1}} =V_{\theta^*}^\top MV_{\theta^*} &\leq v^2\left(\tr(M) + 2 \sqrt{\tr(M^\top M)\log(1/\alpha)}+ 2\norm{M}\log(1/\alpha) \right)\\
   & \leq C_1\cdot v^2\cdot\frac{d+\log(1/\alpha)}{n} 
\end{align*}
for some $C_1 > 0$.
This gives us
\begin{align*}
\gamma_\epsilon\norm{\Delta}^2_{\Sigma_\cD+\lambda I} &\leq \norm{\nabla \hat l_{\cD,\epsilon}(\theta^*)}_{(\Sigma_\cD+\lambda I)^{-1}} \norm{\Delta}_{(\Sigma_\cD+\lambda I)}  + 4 \lambda \gamma_\epsilon B^2\\
& \leq \sqrt{C_1\cdot v^2\cdot\frac{d+\log(1/\alpha)}{n} }\norm{\Delta}_{(\Sigma_\cD+\lambda I)}  + 4 \lambda \gamma_\epsilon B^2~.
\end{align*}
Solving for the above inequality, we get 
\begin{align*}
\norm{\Delta}_{(\Sigma_\cD+\lambda I)}  \leq C_2\cdot \sqrt{\frac{v^2}{\gamma_\epsilon^2}\cdot\frac{d+\log(1/\alpha)}{n} +\lambda B^2 }
\end{align*}
for some constant $C_2 > 0$.
Now note that $\frac{v}{\gamma_\epsilon}=\frac{1}{\gamma}\cdot\frac{e^\epsilon+1}{e^\epsilon -1}$. Hence, we get
\begin{align*}
  \norm{\hat\theta_{\texttt{RR}}-\theta^*}_{(\Sigma_\cD+\lambda I)}  \leq \frac{C}{\gamma}\cdot\frac{e^\epsilon+1}{e^\epsilon -1}  \sqrt{\frac{d+\log(1/\alpha)}{n}}+ C'\cdot \sqrt{\lambda} B,
\end{align*}
for some $C,C' > 0$,
which holds for any $\epsilon \in (0,\infty)$. This completes our proof.

\subsubsection{Logits}
For any $\theta \in \Real^d$, the logits (log-odds) of the probability that the clear-text label $y_i=1$ given $x_i$ is 
\begin{align*}
    \text{logit}(p_{i,1}) \!=\! \log\frac{p_{i,1}}{p_{i,0}}\!=\!\log\frac{\sigma(x_i^\top \theta)}{1-\sigma(x_i^\top \theta)},
\end{align*}
where the same for randomized label $\widetilde y_i=1$ is
\begin{align*}
    \text{logit}( \widetilde p_{i,1})=\log\frac{\widetilde p_{i,1}}{\widetilde p_{i,0}}\! =\! \log \frac{\sigma(x_i^\top \theta)\sigma(\epsilon) \!+\! (1\!-\!\sigma(x_i^\top \theta)) (1\!-\!\sigma(\epsilon))}{(1\!-\!\sigma(x_i^\top \theta))\sigma(\epsilon)\! +\! \sigma(x_i^\top \theta) (1\!-\!\sigma(\epsilon))}~.
\end{align*}
By Jensen's inequality and basics of linear programming, we get
\begin{align*}
    \text{logit}( \widetilde p_{i,1}) \leq \log \left(\frac{\sigma(x_i^\top \theta)\sigma(\epsilon) + (1-\sigma(x_i^\top\theta)) (1-\sigma(\epsilon))}{(1-\sigma(x_i^\top \theta ))^{\sigma(\epsilon)}  \sigma(x_i^\top \theta)^{(1-\sigma(\epsilon))} }\right) \leq \log \left(\frac{\max \lbrace \sigma(\theta^\top x_i), 1-\sigma(\theta^\top x_i)\rbrace}{(1-\sigma(x_i^\top \theta ))^{\sigma(\epsilon)}  \sigma(x_i^\top \theta)^{(1-\sigma(\epsilon))} }\right)~.
\end{align*}
Now, if $p_{i,1} \ge p_{i,0}$, then we have
\begin{align*}
  \text{logit}( \widetilde p_{i,1}) \leq \log \left(\frac{ \sigma(\theta^\top x_i)}{(1-\sigma(x_i^\top \theta ))^{\sigma(\epsilon)}  \sigma(x_i^\top \theta)^{(1-\sigma(\epsilon))} }\right)\leq \sigma(\epsilon)  \log \left(\frac{\sigma(x_i^\top \theta)}{1-\sigma(x_i^\top\theta)} \right)=  \sigma(\epsilon)\cdot\text{logit}(  p_{i,1})~.
\end{align*}
Similarly, observe that
\begin{align*}
    \text{logit}( \widetilde p_{i,0}) \leq \log \left(\frac{(1-\sigma(x_i^\top \theta))\sigma(\epsilon) + \sigma(x_i^\top\theta) (1-\sigma(\epsilon))}{\sigma(x_i^\top \theta )^{\sigma(\epsilon)}  (1-\sigma(x_i^\top \theta))^{(1-\sigma(\epsilon))} }\right) \leq \log \left(\frac{\max \lbrace \sigma(\theta^\top x_i), 1-\sigma(\theta^\top x_i)\rbrace}{\sigma(x_i^\top \theta )^{\sigma(\epsilon)}  (1-\sigma(x_i^\top \theta))^{(1-\sigma(\epsilon))} }\right)~.
\end{align*}
Now, if $p_{i,0} \ge p_{i,1}$, then we have
\begin{align*}
  \text{logit}( \widetilde p_{i,0}) \leq \log \left(\frac{ 1-\sigma(\theta^\top x_i)}{\sigma(x_i^\top \theta )^{\sigma(\epsilon)}  (1-\sigma(x_i^\top \theta))^{(1-\sigma(\epsilon))} }\right)\leq \sigma(\epsilon)  \log \left(\frac{1-\sigma(x_i^\top \theta)}{\sigma(x_i^\top\theta)} \right)=  \sigma(\epsilon)\cdot\text{logit}(  p_{i,0})~.
\end{align*}
Since $\sigma(\epsilon) \!\in\! (1/2,1)$ for any $\epsilon > 0$, this impies that whenever $y_i$ is more likely to occur than $1-y_i$ in the clear-text, the log-odds of predicting $y_i$ under $\epsilon$-randomization given by \eqref{eq:RR} is at most $\sigma(\epsilon)$-th fraction of the corresponding log-odds in the clear-text.

Therefore, we work with the predicted \emph{scores} of randomized labels:
\begin{align*}
 \! \hat p_{i,1}\!=\!\!\frac{ \sigma(x_i^\top \theta)^{\sigma(\epsilon)}}{ (1\!-\!\sigma(x_i^\top \theta))^{(1\!-\!\sigma(\epsilon))}},\quad \hat p_{i,0}\!=\!\!\frac{ (1\!-\!\sigma(x_i^\top \theta))^{\sigma(\epsilon)}} {\sigma(x_i^\top \theta)^{(1\!-\!\sigma(\epsilon))}}~, 
\end{align*}
which have the property
\begin{align*}
    \log\frac{\hat p_{i,1}}{\hat p_{i,0}}\!=\!\log\frac{\sigma(x_i^\top \theta)}{1-\sigma(x_i^\top \theta)} \!=\! \text{logit}(p_{i,1})~, 
\end{align*}
i.e., the log-odds of predicting $y_i$ under $\epsilon$-randomization is same as the corresponding log-odds in the clear-text.
\subsection{Proof of Theorem~\ref{thm:sgd_rr}}
\label{proof:sgd-rr}

\begin{algorithm}[!t]
  \caption{SGD with Randomized Response}
  \label{alg:LabelDP-main}
\begin{algorithmic}[1]
  \STATE {\bfseries Parameters:} privacy budget $\epsilon$; i.i.d dataset $\cD = (x_i,y_i)_{i=1}^n$; parameter space $\Theta_B$; learning rate $(\eta_t)_{t \ge 1}$.
  \STATE {\bfseries Initialize:} $\hat\theta_1 = 0$.
 \FOR{$t\!=\!1, \ldots, n$}
    \STATE Take data point $(x_t, y_t)$ from the dataset $\cD$.
    \STATE Let $\widetilde{y}_t$ be the output of RR mechanism on $y_t$, i.e., 
    \begin{align*}
        \prob{\tilde{y}_t = y_t} = \frac{e^{\epsilon}}{1+e^{\epsilon}} \text{ and }  \prob{\tilde{y}_t \neq y_t} = \frac{1}{1+e^{\epsilon}}~. 
    \end{align*}
    \STATE Compute the gradient \begin{align*}
        \hat{g}_t = \frac{\sum_{y\in \lbrace 0,1\rbrace} \nabla_{\hat\theta_t} \log  p_{t,y}}{e^{\epsilon} + 1}- \nabla_{\hat\theta_t} \log  p_{t,\widetilde y_t} ~,
\end{align*}
where $p_{t,y}$ denotes the probability of observing $y \in \lbrace 0,1\rbrace$ at round $t$, see \eqref{eq:gen-np}.
    \STATE Update the estimate $\hat\theta_{t+1} = \Pi_{\Theta_B}(\hat\theta_t - \eta_t \hat{g}_t)$
  \ENDFOR
  \STATE Output $\hat\theta_{\text{SGD-RR}}=\hat\theta_{n+1}$.
\end{algorithmic}
\end{algorithm}

We divide the proof of Theorem~\ref{thm:sgd_rr} into the following steps. For ease of presentation, the complete algorithm for computing $\hat\theta_{\texttt{SGD-RR}}$ is given in Algorithm~\ref{alg:LabelDP-main}.

\textbf{Step 1:} For each $t \ge 1$, we aim to show that there exists some constants $\lambda$, $G$ and random variable $\hat{z}_t$ such that  
\begin{align}
\label{eq:goal1}
    \norm{\widehat\theta_{t+1} - \theta^*}^2 \le (1-2/t)\norm{\widehat\theta_{t} - \theta^*}^2 + \frac{2}{\lambda t} \inner{\hat{z}_t}{\theta_t - \theta*} + \left(\frac{G}{\lambda t}\right)^2.
\end{align}
To this end, first recall that that the gradient at round $t$ is given by
\begin{align*}
        \hat{g}_t = \frac{\sum_{y\in \lbrace 0,1\rbrace} \nabla_{\hat\theta_t} \log  p_{t,y}}{e^{\epsilon} + 1}- \nabla_{\hat\theta_t} \log  p_{t,\widetilde y_t} ~,
\end{align*}
where $p_{t,y}, y \in \lbrace 0,1\rbrace$ denotes the probability of observing $y$ at round $t$, see \eqref{eq:gen-np}.

Now, we define $\hat{z}_t := \mathbb{E}[\hat{g}_t | \cF_{t-1}] - \hat{g}_t$, where $\cF_{t-1}=\sigma\big(\lbrace x_s,y_s,\widetilde y_s\rbrace_{s=1}^{t-1}\big)$ is the $\sigma$-algebra generated by all the random variables up to and including round $t-1$. This conditioning is necessary since $\widehat\theta_t$ depends on randomness till round $t-1$.
Then, we have 
\begin{align}
    \norm{\widehat\theta_{t+1} - \theta^*}^2 &= \norm{\Pi_{\Theta_B}(\widehat\theta_t - \eta_t \hat{g}_t) - \theta^*}^2\nonumber\\
    &\le \norm{\widehat\theta_t - \eta_t \hat{g}_t - \theta^*}^2\nonumber\\
    &= \norm{\widehat\theta_t - \theta^*}^2 - 2 \eta_t \inner{\hat{g}_t}{\widehat\theta_t - \theta^*} + \eta_t^2 \norm{\hat{g}_t}^2 \nonumber\\
    &= \norm{\widehat\theta_t - \theta^*}^2 -2 \eta_t \inner{\mathbb{E}[\hat{g}_t | \cF_{t-1}] }{\widehat\theta_t - \theta^*} + 2\eta_t \inner{\hat z_t}{\hat\theta_t - \theta^*} + \eta_t^2 \norm{\hat{g}_t}^2\label{eq:goal1-temp}
\end{align}
where the last equality holds by definition of $\hat{z}_t$, i.e., $\hat{g}_t = \mathbb{E}[\hat{g}_t | \cF_{t-1}] -\hat  z_t$.

To bound the above, we need to study the term $\inner{\mathbb{E}[\hat{g}_t | \cF_{t-1}] }{\hat\theta_t - \theta^*}$. First note that $\hat{g}_t$ is an unbiased and scaled estimate of the clear-text gradient $g_t = -\nabla_{\hat\theta_t} \log  p_{t,y_t}$ as 
\begin{align*}
     \mathbb{E}[\hat{g}_t | \cF_{t-1},x_t,y_t]
    &=  \frac{\sum_{y\in \lbrace 0,1\rbrace} \nabla_{\hat\theta_t} \log  p_{t,y}}{e^{\epsilon} + 1} - \left( \frac{e^{\epsilon}}{e^{\epsilon}+1} \nabla_{\hat\theta_t} \log  p_{t,y_t} + \frac{1}{e^{\epsilon}+1} \nabla_{\hat\theta_t} \log  p_{t,1-y_t}  \right)\\
    &=-\frac{e^\epsilon-1}{e^\epsilon+1} \nabla_{\hat\theta_t} \log  p_{t,y_t} = (2\sigma(\epsilon)-1)g_t~.
\end{align*}
Then, by tower property of conditional expectation, we have
\begin{align*}
    \mathbb{E}[\hat{g}_t | \cF_{t-1}] = \mathbb{E}[\mathbb{E}[\hat{g}_t | \cF_{t-1},x_t,y_t] | \cF_{t-1}] = (2\sigma(\epsilon)-1)\mathbb{E}[g_t| \cF_{t-1}]= (2\sigma(\epsilon)-1)\mathbb{E}[(\sigma(x_t^\top \hat\theta_t)-y_t)x_t| \cF_{t-1}]~,
\end{align*}
where the final equality holds by definition of $g_t$. One more application of tower property gives us
\begin{align*}
    \mathbb{E}[(\sigma(x_t^\top \hat\theta_t)-y_t)x_t| \cF_{t-1}] = \mathbb{E}[\mathbb{E}[(\sigma(x_t^\top \hat\theta_t)-y_t)x_t|\cF_{t-1},x_t]| \cF_{t-1}]= \mathbb{E}[(\sigma(x_t^\top \hat\theta_t)-\sigma(x_t^\top \hat\theta^*))x_t|\cF_{t-1}]~.
\end{align*}
Since $\hat\theta_t$ is deterministic given $\cF_{t-1}$, we can bound
$\inner{\mathbb{E}[\hat{g}_t | \cF_{t-1}] }{\hat\theta_t - \theta^*}$ using the above two equation as
\begin{align}
    \inner{\mathbb{E}[\hat{g}_t | \cF_{t-1}] }{\hat\theta_t - \theta^*} &= (2\sigma(\epsilon)-1)\mathbb{E}[\inner{(\sigma(x_t^{\top}\hat \theta_t) - \sigma((x_t^{\top} \theta^*) )x_t}{\hat\theta_t - \theta^*} |\cF_{t-1}]\nonumber\\
    &\gep{a} \gamma(2\sigma(\epsilon)-1) \mathbb{E}[ (x_t^{\top} (\hat\theta_t - \theta^*))^2 |\cF_{t-1}]\nonumber\\
    & \gep{b} \gamma_\epsilon (\hat\theta_{t} - \theta^*)^{\top}\mathbb{E}[x_t x_t^{\top} |\cF_{t-1}](\hat\theta_{t} - \theta^*)\nonumber\\
    &\gep{c}\gamma_\epsilon \kappa \norm{\hat\theta_t -\theta^*}^2~.\label{eq:gt-bound}
\end{align}
Here (a) holds by mean-value theorem and by noting that $\sigma'(\theta^\top x_i) \ge \gamma$ for all $\theta \in \Theta_B$ under Assumption~\ref{ass:bound}, where $\gamma = \frac{1}{2 + \exp(-2LB) + \exp(2LB)}$; (b) holds by defining $\gamma_\epsilon = (2\sigma(\epsilon)-1)\gamma$; (c) holds by Assumption~\ref{ass:cov} and by noting that $x_t$ is independent of $\cF_{t-1}$.

Now, plugging~\eqref{eq:gt-bound} into~\eqref{eq:goal1-temp}, yields 
\begin{align*}
     \norm{\hat\theta_{t+1} - \theta^*}^2 &\le \norm{\hat\theta_t - \theta^*}^2(1 -2\eta_t \gamma \kappa) + 2\eta_t \inner{z_t}{\hat\theta_t - \theta^*} + \eta_t^2 \norm{\hat{g}_t}^2\\
     &\lep{a} \norm{\hat\theta_t - \theta^*}^2(1 -2\eta_t \gamma_\epsilon \kappa) + 2\eta_t \inner{z_t}{\hat\theta_t - \theta^*} + \eta_t^2 G^2\\
     &\ep{b}  (1-2/t)\norm{\hat\theta_{t} - \theta^*}^2 + \frac{2}{\lambda t} \inner{\hat{z}_t}{\hat\theta_t - \theta*} + \left(\frac{G}{\lambda t}\right)^2
\end{align*}
where (a) holds by bounding $\norm{\hat{g}_t} \le G := 4L $ under Assumption~\ref{ass:bound}; (b) holds by letting $\lambda := \gamma_\epsilon \kappa$ and $\eta_t:= \frac{1}{\lambda t}$; Hence, we have established~\eqref{eq:goal1}.

\textbf{Step 2:} We aim to show that for all $t\ge 2$
\begin{align}
\label{eq:goal2}
\norm{\hat\theta_{t+1} - \theta^*}^2 \le \frac{2}{\lambda (t-1) t}\sum_{i=2}^t (i-1)\inner{\hat{z}_i}{\hat\theta_i - \theta^*} + \frac{G^2}{\lambda^2 t^2}.
\end{align}

To this end, we basically expand the recursion in~\eqref{eq:goal1} till $t=2$ and simple algebra leads to the result. This step directly follows from~\cite{rakhlin2011making}.

\textbf{Step 3:} We will apply one particular version of Freedman's inequality to control the concentration of $\sum_{i=2}^t (i-1)\inner{\hat{z}_i}{\hat\theta_i - \theta^*}$ in~\eqref{eq:goal2}. In particular, we will apply~\cite[Lemma 3]{rakhlin2011making} to bound this sum of martingale differences for all $t \le n$. This needs to hold for all $t$ since we will rely on induction later. 

To start with, we let $Z_i = \inner{\hat{z}_i}{\hat\theta_i -\theta^*}$. Then, we have the conditional expectation of $Z_i$ given $\cF_{i-1}$ is $\mathbb{E}[Z_i |\cF_{i-1}] = 0$ and the conditional variance is $\text{Var}[Z_i |\cF_{i-1}] \le 4G^2 \norm{\hat\theta_i -\theta^*}^2$, which holds by $\norm{\hat{z}_i} \le 2G$.
Now consider the sum $\sum_{i=2}^t (i-1)\inner{\hat{z}_i}{\theta_i - \theta^*}$ in~\eqref{eq:goal2}. We need to check two conditions: (i) The sum of conditional variance satisfies
\begin{align*}
    \sum_{i=2}^t \text{Var}[(i-1)Z_i |\cF_{i-1}] \le 4G^2 \sum_{i=2}^t (i-1)^2 \norm{\hat\theta_i - \theta^*}^2~;
\end{align*}
(ii) Uniform upper bound on each term satisfies
\begin{align*}
   |(i-1)Z_i| \le 2G (t-1)\norm{\hat\theta_i - \theta^*} \lep{a} \frac{2G^2 (t-1)}{\lambda},
\end{align*}
where (a) comes from~\eqref{eq:gt-bound} and substituting $\lambda = \gamma_\epsilon \kappa$. To see this, by Cauchy-Schwartz inequality, we have $\gamma_\epsilon \kappa \norm{\hat\theta_t - \theta^*}^2 \le G \norm{\hat\theta_t - \theta^*}$, and hence $\norm{\hat\theta_t - \theta^*} \le G/\lambda$ for all $t$. We can then apply~\cite[Lemma 3]{rakhlin2011making} to obtain that for $n \ge 4$ and $\alpha \in (0,1/e)$, with probability at least $1-\alpha$, it holds for all $t \le n$ that
\begin{align}
\label{eq:goal3}
    \sum_{i=2}^t (i-1)Z_i \le 8G \max\left\{\sqrt{\sum_{i=2}^t (i-1)^2 \norm{\hat\theta_i - \theta^*}^2}, \frac{G(t-1)}{\lambda}\sqrt{\log(\log n/\alpha)}\right\}\sqrt{\log(\log n/\alpha)}.
\end{align}

\textbf{Step 4:} Once we obtain~\eqref{eq:goal3}, the remaining step is all about induction and algebra, which follows the same procedures as in~\cite{rakhlin2011making}. After all, we will obtain with probability at least $1-\alpha$ that for all $t \le n$,
\begin{align*}
    \norm{\hat\theta_{t+1} - \theta^*}^2 &\le \frac{(624\log(\log n/\alpha)+1)G^2}{\lambda^2 t}\\
    &=C L^2 \left(\frac{e^{\epsilon}+1}{e^{\epsilon}-1}\right)^2\cdot \frac{\log(\log n/\alpha)+1}{\gamma^2 \kappa^2 t},
\end{align*}
for some absolute constant $C$. Setting $t=n$ completes the proof.

\subsection{Estimation Error under Placket-Luce Model}\label{proof:PL}

Let $s$ be a state and $a_1,\ldots,a_K$ be $K$ actions to be compared at that state. Let the label/preference feedback $y \in \lbrace 1,2,\ldots,K\rbrace$ indicates which action is most preferred by human labeler.
Let $x_{i,j}=\phi(s,a_i)-\phi(s,a_j), 1\le i \neq j \le K$ be the feature difference between actions $a_i$ and $a_j$ at state $s$. Define the population covariance matrix
\begin{align*}
    \Sigma_{i,j} = \mathbb{E}_{s \sim \rho(\cdot), (a_1,\ldots,a_K)\sim \mu(\cdot | s)} \left[ x_{i,j} x_{i,j}^{\top} \right]~.
\end{align*}
\begin{assumption}[Coverage of feature space]
\label{ass:cov_PL}
The data distributions $\rho,\mu$ are such that $\lambda_{\min}(\Sigma_{i,j}) \ge \kappa$ for some constant $\kappa \!>\! 0$ for all $1\le i \neq j \le K$.
\end{assumption}
This is a coverage assumption on the state-action feature space. The next result bounds the estimation error of  $\hat\theta_{\texttt{SGD-KRR}}$ in $\ell_2$-norm.

\begin{theorem}[Estimation error of $\hat\theta_{\texttt{SGD-KRR}}$ under Placket-Luce model]
\label{thm:sgd_krr}
    Fix $\alpha \in (0,1/e)$ and $\epsilon > 0$. Then, under the Placket Luce model \eqref{eq:PL} and under Assumptions~\ref{ass:bound} and~\ref{ass:cov_PL} and setting $\eta_t = \frac{1}{\gamma \kappa}$, we have, with probability at least $1-\alpha$, 
    \begin{align*}
\norm{\hat\theta_{\texttt{SGD-KRR}} \!-\! \theta^*}_2 \!\le\! C \!\cdot\! \frac{L}{\gamma \kappa} \!\cdot\!\frac{e^{\epsilon}\!+K-\!1}{e^{\epsilon}\!-\!1}\!\sqrt{ \frac{\log\left(\log(n)/\alpha\right)}{ n}},
    \end{align*}
where $\gamma = \frac{e^{-4LB}}{2}$, $C$ is an absolute constant.
\end{theorem}

\begin{proof}
We will first show that for all $t \ge 1$, the parameter updates satisfy
\begin{align}\label{eq:goal-PL}
 \inner{\mathbb{E}[\hat{g}_t | \cF_{t-1}] }{\hat\theta_t - \theta^*} \ge \gamma_{K,\epsilon}\kappa  \norm{\hat\theta_t -\theta^*}^2~, 
\end{align}
where $\hat{g}_t = \frac{\sum_{y=1}^K  \nabla_{\hat\theta_t} \log  p_{t,y}}{e^{\epsilon} +K- 1}- \nabla_{\hat\theta_t} \log  p_{t,\widetilde y_t}$ is the gradient, $\cF_{t-1}=\sigma\big(\lbrace x_s,y_s,\widetilde y_s\rbrace_{s=1}^{t-1}\big)$ is the $\sigma$-algebra generated by all the random variables up to and including round $t-1$, $\gamma_{K,\epsilon}:=\gamma \frac{e^\epsilon - 1}{e^\epsilon + K -1}$ and $\gamma=e^{-4LB}/2$. Then, one can follow the steps used in the proof of Theorem~\ref{thm:sgd_rr} to derive this result.

Let's now establish \eqref{eq:goal-PL}. To this end, let $\Pi$ be the set of all permutations $\pi : [K] \to [K]$ that denotes a ranking over all $K$ actions given by a human labeler, where $a_{\pi(1)}$ denotes the highest-ranked action. Under the Placket-Luce model, one can compute the probability of observing the permutation $\pi \in \Pi$ as
\begin{align*}
\mathbb{P}_{\theta^*}\!\!\left[\pi|s, a_1,\ldots, a_K\right] \!=\! \prod_{j=1}^K\frac{\exp(\phi(s,a_{\pi(j)})^\top \theta^*)}{\sum_{k'=j}^{K}\exp(\phi(s,a_{\pi(k')})^\top \theta^*) }~.
\end{align*}
Define, with an abuse of notation, $x = (s,a_1,\ldots,a_K)$ and $x_{\pi(j)}=\phi(s,a_{\pi(j)})$ for all $j \in [K]$. This lets us denote for any $\theta \in \Real^d$:
\begin{align*}
\mathbb{P}_{\theta}\!\left[\pi|x\right] \!=\!  \prod_{j=1}^K \mathbb{P}_{\theta}\!\left[\pi(j)|x\right]~,\,\, \text{where}\,\, \mathbb{P}_{\theta}\!\left[\pi(j)|x\right]=\frac{\exp(x_{\pi(j)}^\top \theta)}{\sum_{k'=j}^{K}\exp(x_{\pi(k')}^\top \theta) }~.
\end{align*}
The negative log-likelihood (log-loss) of predicting the the highest-ranked action $a_{\pi(1)}$ given $x$ is
\begin{align*}
    l_\theta(a_{\pi(1)},x) := -\log \mathbb{P}_{\theta}\!\left[\pi(1)|x\right] = -\log\frac{\exp(x_{\pi(1)}^\top \theta)}{\sum_{k'=1}^{K}\exp(x_{\pi(k')}^\top \theta) }~.
\end{align*}
The expected log-loss takes the form
\begin{align*}
G_\theta(x):=\mathbb{E}_{\pi \sim \mathbb{P}_{\theta^*}[\cdot |x]}\big[l_\theta(a_{\pi(1)},x)\big] = \sum_{\pi \in \Pi} \mathbb{P}_{\theta^*}\!\left[\pi|x\right]l_{\theta}\!\left(a_{\pi(1)},x\right)~.
\end{align*}
This yields the following:
\begin{align*}
 \nabla^2 G_\theta(x) = \sum_{\pi \in \Pi} \mathbb{P}_{\theta^*}\!\left[\pi|x\right]\nabla^2  l_{\theta}(a_{\pi(1)},x)~. 
\end{align*}
Note that the following holds  \citep{zhu2023principled}:
\begin{align*}
    \nabla l_{\theta}(a_{\pi(1)},x) &=  \sum_{k=1}^K \frac{\exp(x_{\pi(k)}^\top \theta)}{\sum_{k'=1}^{K}\exp(x_{\pi(k')}^\top \theta) } \left( x_{\pi(1)}-x_{\pi(k)}\right)~,\\
    \nabla^2 l_{\theta}(a_{\pi(1)},x) &=  \sum_{k=1}^K \sum_{k'=1}^K\frac{\exp(x_{\pi(k)}^\top \theta) \cdot \exp(x_{\pi(k')}^\top \theta)}{2\left(\sum_{k'=1}^{K}\exp(x_{\pi(k')}^\top \theta) \right)^2} \left( x_{\pi(k)}-x_{\pi(k')}\right) \left( x_{\pi(k)}-x_{\pi(k')}\right)^\top~.
\end{align*}
Under Assumption~\ref{ass:bound}, we have $-LB \le \phi(s,a)^\top \theta \le LB$ for all $\theta \in \Theta_B$. Define $x_{\pi,k,k'}=x_{\pi(k)}-x_{\pi(k')}$ for all $k,k' \in [K]$. Then, for any $v \in \Real^d$ and $\theta \in \Theta_B$, we have
\begin{align*}
    v^\top  \nabla^2 l_\theta(a_{\pi(1)},x) v \ge \frac{e^{-4LB}}{2}\cdot v^\top\left( \frac{1}{K^2}\sum_{k=1}^K \sum_{k'=1}^K  x_{\pi,k,k'}x_{\pi,k,k'}^\top\right) v~.
\end{align*}
Define the matrix
\begin{align*}
\Sigma(\pi,x):=\frac{1}{K^2}\sum_{k=1}^K \sum_{k'=1}^K  x_{\pi,k,k'}x_{\pi,k,k'}^\top~.
\end{align*}
Then, for $\theta \in \Theta_B$ and $\pi \in \Pi$, the loss function $l_\theta(a_{\pi(1)},x)$ is $\gamma=\frac{e^{-4LB}}{2}$ strongly convex w.rt. the semi-norm $\norm{\cdot}_{\Sigma(\pi,x)}$. This further implies that $G_\theta(x)$ is $\gamma$ strongly convex w.r.t. the semi-norm $\norm{\cdot}_{\Sigma(x)}$, where $\sigma(x):=\sum_{\pi \in \Pi} \mathbb{P}_{\theta^*}\!\left[\pi|x\right]\Sigma(\pi,x)$.

Since $\theta^* \in \Theta_B$, we have from definition of strong convexity,
\begin{align*}
    G_{\theta^*}(x) \ge G_{\theta}(x) + \inner{\nabla G_{\theta}(x)}{\theta^*-\theta} +\frac{\gamma}{2}\norm{\theta -\theta^*}^2_{\Sigma(x)} \implies \inner{\nabla G_{\theta}(x)}{\theta-\theta^*} \ge G_{\theta}(x) - G_{\theta^*}(x) +\frac{\gamma}{2}\norm{\theta -\theta^*}^2_{\Sigma(x)}  ~.
\end{align*}
Since $\theta^* \in \argmin_{\theta \in \Theta_B} G_{\theta}(x)$, we have from definition of first-order optimality of convex functions,
\begin{align*}
    G_\theta(x) - G_{\theta^*}(x) \ge \inner{\nabla G_{\theta^*}(x)}{\theta-\theta^*} +\frac{\gamma}{2}\norm{\theta -\theta^*}^2_{\Sigma(x)} \ge \frac{\gamma}{2}\norm{\theta -\theta^*}^2_{\Sigma(x)}~.
\end{align*}
Combining the above, we have for any $\theta \in \Theta_B$:
\begin{align*}
    \inner{\nabla G_{\theta}(x)}{\theta-\theta^*} \ge \gamma\norm{\theta -\theta^*}^2_{\Sigma(x)}\implies \inner{\mathbb{E}_{\pi \sim \mathbb{P}_{\theta^*}}\big[\nabla l_\theta(a_{\pi(1)},x)|x\big]}{\theta-\theta^*} \ge \gamma\norm{\theta -\theta^*}^2_{\Sigma(x)}~. 
\end{align*}
Now, taking expectation over $x \sim (\rho \times \mu)$, we have 
\begin{align*}
\inner{\mathbb{E}_{x\sim (\rho \times \mu),\pi \sim \mathbb{P}_{\theta^*}[\cdot | x]}\big[\nabla l_\theta(a_{\pi(1)},x)\big]}{\theta-\theta^*} &\ge \gamma(\theta -\theta^*)^\top \mathbb{E}_x\left[\Sigma(x)\right] (\theta -\theta^*)\\
& = \gamma  (\theta -\theta^*)^\top \mathbb{E}_x\left[\sum_{\pi \in \Pi} \mathbb{P}_{\theta^*}[\pi|x] \Sigma(\pi,x)\right] (\theta -\theta^*)~.
\end{align*}


Note that, by the coverage Assumption~\ref{ass:cov_PL}, we have $\mathbb{E}_x\left[\Sigma(\pi,x) \right] \ge \kappa$ for all $\pi \in \Pi$.
This yields for any $v \in \Real^d$,
\begin{align*}
   \forall \pi \in \Pi,\, v^\top \mathbb{E}_x\left[\Sigma(\pi,x) \right] v \ge \kappa \norm{v}^2 \implies \mathbb{E}_x\left[\min_{\pi \in \Pi}v^\top \Sigma(\pi,x) v \right]  \ge \kappa\norm{v}^2~.
\end{align*}
This further yields 
\begin{align*}
    v^\top \mathbb{E}_x\left[\sum_{\pi \in \Pi} \mathbb{P}_{\theta^*}[\pi|x] \Sigma(\pi,x) \right] v &= \mathbb{E}_x\left[\sum_{\pi \in \Pi} \mathbb{P}_{\theta^*}[\pi|x] v^\top \Sigma(\pi,x) v \right]\\
    & \ge \mathbb{E}_x\left[\min_{\pi \in \Pi}  v^\top \Sigma(\pi,x) v \right] \ge \kappa\norm{v}^2~.
\end{align*}
This implies for any $\theta \in \Theta_B$, the following:
\begin{align}\label{eq:convex-coverage}
\inner{\mathbb{E}_{x\sim (\rho \times \mu),\pi \sim \mathbb{P}_{\theta^*}[\cdot |x]}\big[\nabla l_\theta(a_{\pi(1)},x)\big]}{\theta-\theta^*} \ge  \gamma \kappa  \norm{\theta-\theta^*}^2~.
\end{align}
Now, let $\pi_t$ be the permutation (ranking) given by human labeler at round $t$, i.e. $\pi_t(1)=y_t$, and $\widetilde \pi_t$ be the (noisy) ranking after randomization by KRR mechanism~\eqref{eq:KRR}, i.e. $\widetilde\pi_t(1)=\widetilde y_t$. Note that, we have
\begin{align*}
   \prob{\widetilde{\pi}_t(1) = \pi_t(1)} = \frac{e^{\epsilon}}{e^{\epsilon}+K -1} ~,\,\prob{\widetilde\pi_t(1) = y} = \frac{1}{e^{\epsilon}+K-1}~, \forall y \neq \pi_t(1)~. 
\end{align*}
Using this, we can re-write the gradient as 
\begin{align*}
        \hat{g}_t  = \frac{\sum_{y=1}^K  \nabla_{\hat\theta_t} \log  p_{t,y}}{e^{\epsilon} +K- 1}- \nabla_{\hat\theta_t} \log  p_{t,\widetilde \pi_t(1)} ~,
\end{align*}
where $p_{t,y}, y \in [K]$ is the probability of observing $y$ at round $t$, see \eqref{eq:PL}. 
Then, we have 
\begin{align*}
     \mathbb{E}[\hat{g}_t | \cF_{t-1},x_t,y_t]
    &=  \frac{\sum_{y\in \lbrace 0,1\rbrace} \nabla_{\hat\theta_t} \log  p_{t,y}}{e^{\epsilon} +K- 1} - \left( \frac{e^{\epsilon}}{e^{\epsilon}+1} \nabla_{\hat\theta_t} \log  p_{t,y_t} + \frac{1}{e^{\epsilon}+1} \nabla_{\hat\theta_t} \log  p_{t,1-y_t}  \right)\\
    &=-\frac{e^\epsilon-1}{e^\epsilon+K-1} \nabla_{\hat\theta_t} \log  p_{t,y_t} = \frac{e^\epsilon-1}{e^\epsilon+K-1}  \nabla l_{\hat \theta_t}(a_{\pi_t(1)},x_t) ~.
\end{align*}
Since $\hat\theta_t$ is deterministic given $\cF_{t-1}$, by tower property of conditional expectation, we have
\begin{align*}
   \inner{\mathbb{E}[\hat{g}_t | \cF_{t-1}] }{\hat\theta_t - \theta^*} &= \frac{e^\epsilon - 1}{e^\epsilon + K -1} \mathbb{E}[\inner{\nabla l_{\hat \theta_t}(a_{\pi_t(1)},x_t)}{\hat\theta_t - \theta^*} |\cF_{t-1}]\nonumber\\
    &\ge \gamma \kappa \frac{e^\epsilon - 1}{e^\epsilon + K -1}\norm{\hat\theta_t -\theta^*}^2~,
\end{align*}
where the last step follows from \eqref{eq:convex-coverage} and by noting that $x_t,\pi_t$ are independent of $\cF_{t-1}$. Defining $\gamma_{K,\epsilon}:=\gamma \frac{e^\epsilon - 1}{e^\epsilon + K -1} $, we get \eqref{eq:goal-PL}. This completes our proof.
\end{proof}

\section{Additional Details on Section~\ref{sec:central}}
\label{app:central}
\subsection{Proof of Theorem~\ref{thm:central-lower}}
\label{proof:central-lower}
Before presenting the proof, let us introduce the following useful lemma. For the label-DP in the central model, we will leverage the DP version of Assouad’s lemma in~\citet{acharya2021differentially}, which is re-stated as follows\footnote{We correct some constant factor error in the original statement.}. 
\begin{lemma}[Assouad's lemma for central DP]
\label{lem:DP-assouad}
Let the same conditions of Lemma~\ref{lem:assouad} hold. If for all $i \in [d]$, there exists a coupling $(X,Y)$ between $P_{+i}$ and $P_{-i}$ with $\ex{d_{\mathrm{Ham}}(X,Y)} \le D$ for some $D \ge 0$, then 
\begin{align*}
     R_c(\cP, \rho,\epsilon,\delta) \ge \frac{d \tau}{2\alpha} \cdot \left(0.9e^{-10\epsilon D} - 10D\delta\right) .
\end{align*}
\end{lemma}

Now, we are well-prepared to prove Theorem~\ref{thm:central-lower}.
\begin{proof}[Proof of Theorem~\ref{thm:central-lower}]
    First note that the non-private part is the same as before. Thus, we only need to focus on the second private part. 

    Choose some $\Delta >0$ and for each $e \in \cE_d = \{\pm 1\}^d$,  let $\theta_e = \Delta e$. Now we need to check the two conditions in Lemma~\ref{lem:assouad}. First note that $\rho = \norm{\cdot}_2^2$ satisfies $2$-triangle inequality, i.e., $\alpha = 2$. Also, note that for any $u, v \in \cE_d$, $\norm{\theta_u - \theta_v}_2^2 = 4\Delta^2 \sum_{i=1}^d \mathds{1}(u_i \neq v_i)$, i.e., $\tau = 2\Delta^2$. Thus, let $P_{+i}^n$ be the product distribution of $P_{+i}$ 
and similarly for $P_{-i}^n$, then by Lemma~\ref{lem:DP-assouad}, we have 
    \begin{align*}
R_c(\cP_{\text{log}},\norm{\cdot}_2^2,\epsilon,\delta) &\ge  \frac{d\Delta^2}{2} \left(0.9e^{-10\epsilon D} - 10D\delta\right)\\
         &\gep{a} \frac{d\Delta^2}{2}\left(0.9 - 10 D(\epsilon + \delta)\right)
    \end{align*}
    where $D$ is the bound on the expected hamming distance between $(X,Y)$, which is a coupling between $P_{+i}^n$ and $P_{-i}^n$; (a) holds by the fact that $e^x \ge 1+x$.

     Thus, it remains to determine $D$ in our case.  That is, we need to bound the expected hamming distance between two product distributions $P_{+i}^n$ and $P_{-i}^n$. Note that for the lower bound, it suffices to consider  $ x_k = x \in \Real^d$ with $\norm{x}_{\infty} \le 1$ for all $k \in [n]$. In this case, by the standard result on maximal coupling, we have that for the maximal coupling $(X,Y)$ between $P_{+i}^n$ and $P_{-i}^n$,
    \begin{align*}
        \ex{d_{\mathrm{Ham}}(X,Y)} =  n \norm{P_{+i} - P_{-i}}_{\mathrm{TV}}.
    \end{align*}
    Now, it remains to bound the TV-distance. To this end, by~\eqref{eq:rewrite} and joint convexity of TV distance, we have 
    \begin{align*}
        \norm{P_{+i} - P_{-i}}_{\mathrm{TV}} &= \norm{\frac{1}{|\cE_d|} \sum_{e\in \cE_d} P_{e,+i} - P_{e,-i}}_{\mathrm{TV}}\\
        &\le \frac{1}{|\cE_d|} \sum_{e \in \cE_d} \norm{P_{e,+i} - P_{e,-i} }_{\mathrm{TV}}\\
        &\le \max_{e \in \cE_d, i\in[d]} \norm{P_{e,+i} - P_{e,-i} }_{\mathrm{TV}}\\
        &= \max_{e \in \cE_d, i\in[d]} \norm{P_{e} - P_{\bar{e}^i} }_{\mathrm{TV}},
    \end{align*}
    where recall that $\bar{e}^i$ is a vector in $\cE_d$ that flips the $i$-th coordinate of $e$.
To proceed, for any $i \in [d]$, by Pinsker's inequality, we have 
\begin{align*}
    \norm{P_{e} - P_{\bar{e}^i} }_{\mathrm{TV}}^2 &\le \frac{1}{4} \left(\kl{P_{e}}{ P_{\bar{e}^i}}  + \kl{ P_{\bar{e}^i}}{P_{e}}\right) \lep{a} \Delta^2,
\end{align*}
where (a) follows from Claim~\ref{clm:kl} and the choice of $x$ such that $\norm{x}_{\infty} \le 1$. Thus, putting everything together, yields that 
\begin{align*}
        \ex{d_{\mathrm{Ham}}(X,Y)} =  n \norm{P_{+i} - P_{-i}}_{\mathrm{TV}} \le n \Delta := D.
    \end{align*}
With this value of $D$, we finally obtain that 
\begin{align*}
R_c(\cP_{\text{log}},\norm{\cdot}_2^2,\epsilon,\delta) \ge \frac{d\Delta^2}{2} \left(0.9 - 10 n \Delta (\epsilon + \delta)\right).
\end{align*}
Thus, choosing $\Delta = \frac{0.04}{n(\epsilon + \delta)}$, we obtain that 
\begin{align*}
R_c(\cP_{\text{log}},\norm{\cdot}_2^2,\epsilon,\delta) \ge c \cdot \frac{d}{n^2(\epsilon + \delta)^2},
\end{align*}
for some universal constant $c$. Finally, combined with the non-private part, we have finished the proof.
\end{proof}

\subsection{Proof of Theorem~\ref{thm:central-priv}}
\label{proof:central-priv}
Before we present the proof, we first highlight some differences between our proof and the one in~\cite{kifer2012private}. In particular, we note that one cannot simply follow the one in~\cite{kifer2012private} as there exists a gap in Lemma 16 of~\cite{kifer2012private} due to non-independence. Thus, we carefully handle this subtlety under our model.  We also explicitly write down the two-step procedures of Successive Approximation to handle the minimization over a constrained set. 

\begin{algorithm}[!t]
  \caption{Objective Perturbation with Gaussian Noise}
  \label{alg:objP}
\begin{algorithmic}[1]
  \STATE {\bfseries Parameters:} privacy budget $\epsilon > 0,\delta \in (0,1)$; regularization parameter $\beta$; i.i.d dataset $\cD = (x_i,y_i)_{i=1}^n$; parameter space $\Theta_B$; log loss $\ell$
  \STATE Sample $w \in \cN(0,\sigma^2 I)$
  \STATE Return  $\hat{\theta}_{\texttt{obj}} = \argmin_{\theta \in \Theta_B} {l}_{\cD}(\theta) + \frac{\beta}{2n} \norm{\theta}_2^2 +  \frac{w^{\top} \theta}{n}$, where ${l}_{\cD}(\theta) = \frac{1}{n}\sum_{i=1}^n \ell(\theta,(x_i,y_i))$
\end{algorithmic}
\end{algorithm}

Now, we are ready to present the proof.
\begin{proof}[Proof of Theorem~\ref{thm:central-priv}]
    Our goal is to show that $\hat\theta_{\texttt{obj}}$ is $(\epsilon,\delta)$-label DP in the central model, where
    \begin{align}
    \label{eq:goal-dp}
        \hat{\theta}_{\texttt{obj}} = \argmin_{\theta \in \Theta_B} {l}_{\cD}(\theta) + \frac{\beta}{2n} \norm{\theta}_2^2 +  \frac{w^{\top} \theta}{n}~.
    \end{align}
    To this end, we will first use Successive Approximation (Theorem 1 in~\cite{kifer2012private}), which allows us to only focus on the following sequence of unconstrained problems (indexed by $i \in \mathbb{N}$).
    \begin{align}
    \label{eq:uncons-dp}
    \hat{\theta}_{\texttt{obj}}^{(i)} = \argmin_{\theta \in \Real^d} {l}_{\cD}(\theta) + \frac{\beta}{2n} \norm{\theta}_2^2 +  \frac{w^{\top} \theta}{n} + \frac{if(\theta)}{n},
    \end{align}
    where $f(\theta) = \min_{z \in \Theta_B} \norm{\theta - z}_2$, which is a convex function (but not necessarily differentiable everywhere). The technique of Successive Approximation (SA) says that it suffices to show that for each $i$, the computation in~\eqref{eq:uncons-dp} is $(\epsilon,\delta)$-label DP. To show this, we will have to use SA again as $f(\theta)$ in~\eqref{eq:uncons-dp} is not differentiable everywhere. To handle this, for each $i$, we will consider another sequence of problems (indexed by $j \in \mathbb{N}$) as follows 
    \begin{align}
    \label{eq:smooth-dp}
    \hat{\theta}_{\texttt{obj}}^{(i,j)} = \argmin_{\theta \in \Real^d} {l}_{\cD}(\theta) + \frac{\beta}{2n} \norm{\theta}_2^2 +  \frac{w^{\top} \theta}{n} + \frac{1}{n}r^{(i,j)}(\theta),
    \end{align}
    where $r^{(i,j)}(\theta)$ be the convolution between $if(\theta)$ and $K_j$ (defined in Eq.(5) of~\cite{kifer2012private}). Now, we have $r^{(i,j)}(\theta)$ is differentiable everywhere and convex. Thus, it only remains to show that the computation in~\eqref{eq:smooth-dp} is $(\epsilon,\delta)$-label DP, for all $i,j$.

    Fix a pair $(i,j)$, we simplify notation in~\eqref{eq:smooth-dp} by focusing on the following problem. 
    \begin{align}
    \label{eq:simplify}
        \widetilde{\theta}_{\cD} = \argmin_{\theta \in \Real^d} {l}_{\cD}(\theta) + \frac{\beta}{2n} \norm{\theta}_2^2 +  \frac{w^{\top} \theta}{n} + \frac{1}{n}r(\theta).
    \end{align}

    \textbf{Step 1:} Establish the PDF for $\widetilde{\theta}_{\cD}$.
    
    By differentiability, we have $n \nabla_{\theta} {l}_{\cD}(\widetilde{\theta}_{\cD}) + \beta \widetilde{\theta}_{\cD} +  w  + \nabla r(\widetilde{\theta}_{\cD})= 0$. Define $\psi_{\cD}(\theta):= n \nabla_{\theta} {l}_{\cD}({\theta}) + \beta {\theta} + \nabla r(\theta)$ . 
    By change of random variables and $w \sim \cN(0,\sigma^2 I_d)$, we have that the probability density of $\widetilde{\theta}_{\cD}$ is given by for $t \in \Real^d$ 
     \begin{align}
    \label{eq:pdf}
        f_{\widetilde{\theta}_{\cD}}(t) = C \underbrace{\exp\left( - \frac{\norm{\psi_{\cD}(t)}_2^2}{2\sigma^2}\right)}_{\cT_{1,\cD}}\cdot \underbrace{\left| \det \left[ \frac{d \psi_{\cD} (\theta) }{d \theta} |_{\theta = t} \right]\right|}_{\cT_{2,\cD}}
    \end{align}
    where we use the fact that if $X$ has density function $f_X$,  $Y = H(X)$ for some bijective, differentiable function $H$, the $Y$ has density 
    \begin{align*}
        f_Y(y) = f_X(H^{-1}(y))\left| \det \left[ \frac{d H^{-1}(z)}{d z} |_{z = y} \right]\right|.
    \end{align*}
    Note that here the bijective relation holds by the strong convexity thanks to the regularization term $\beta > 0$.

     \textbf{Step 2:} Bound the PDF ratio under two neighboring datasets.
    
    By definition of DP, it suffices to show that for all $t \in \Real^d$, with probability at least $1-\delta$
    \begin{align*}
      e^{-\epsilon}f_{\widetilde{\theta}_{\cD'}}(t) \le  f_{\widetilde{\theta}_{\cD}}(t) \le e^{\epsilon} f_{\widetilde{\theta}_{\cD'}}(t),
    \end{align*}
    for all neighboring datasets $D,D'$. To this end, we first look at the ratio of $\cT_{2,\cD} / \cT_{2,\cD'}$. Note that the matrix inside the determinant in~\eqref{eq:pdf} is the Hessian of $l_{\cD}$ plus some common terms, given by 
    \begin{align*}
        \nabla^2 r(\theta)|_{\theta = t} +  {\beta} I + \nabla^2 {l}_{\cD} (\theta) |_{\theta = t} =  \nabla^2 r(\theta)|_{\theta = t} + {\beta} I + \sum_{i=1}^n \sigma(x_i^{\top}t)(1-\sigma(x_i^{\top}t)) x_i x_i^{\top},
    \end{align*}
    which does not depend on labels $\{y_i\}_{i=1}^n$. Thus, $\cT_{2,\cD} / \cT_{2,\cD'} = 1$.

    Now, we turn to the ratio of $\cT_{1,\cD} / \cT_{1,\cD'}$. In particular, we have 
    \begin{align}
    \label{eq:T1-ratio}
        \frac{\cT_{1,\cD}}{\cT_{1,\cD'}} &= \exp\left(\frac{\norm{\psi_{\cD'}(t)}_2^2 - \norm{\psi_{\cD}(t)}_2^2}{2\sigma^2}\right)\nonumber\\
        &= \exp\left(\frac{2\inner{\psi_{\cD}(t)}{\psi_{\cD'(t)} - \phi_{\cD}(t)} + \norm{\phi_{\cD'}(t) - \psi_{\cD}(t)}_2^2}{2\sigma^2}\right)\nonumber\\
         &= \exp\left(\frac{2\inner{-w}{\psi_{\cD'(t)} - \psi_{\cD}(t)} + \norm{\psi_{\cD'}(t) - \psi_{\cD}(t)}_2^2}{2\sigma^2}\right).
    \end{align}
    where we know that $\psi_{\cD}(t) = -w$, which is distributed according to a normal. However, one needs to be careful here to show that  $\phi_{\cD'}(t) - \phi_{\cD}(t)$ is \emph{independent} of $w$ so that one can claim that the inner product is also distributed according to a normal. In fact, this is not true in general\footnote{This is why Lemma 16 in~\cite{kifer2012private} does not hold in general.}! Fortunately, in our case, for two neighboring datasets $\cD, \cD'$ that differs only in $y_j, y'_j$  we have 
    \begin{align*}
        \psi_{\cD'}(t) - \psi_{\cD}(t) &=  \left(\frac{1}{1+\exp(-\inner{x_j}{t})} - y'_j\right) x_j - \left(\frac{1}{1+\exp(-\inner{x_j}{t})} - y_j\right) x_j\\
        &= x_j y_j - x_j y'_j,
    \end{align*}
    which is independent of the sampled noise $w$. Thus, we now can safely follow a similar approach in~\citet{kifer2012private}. That is, by the concentration of normal distribution, we have with probability at least $1-\delta$, 
    \begin{align*}
        |\inner{-w}{\psi_{\cD'}(t) - \psi_{\cD}(t)}| \le \norm{\psi_{\cD'}(t) - \psi_{\cD}(t)}\sigma \sqrt{2\log(2/\delta)}.
    \end{align*}
    Meanwhile, we have $\norm{\psi_{\cD'}(t) - \psi_{\cD}(t)} \le 2L$ by Assumption~\ref{ass:bound}. Putting everything back to~\eqref{eq:T1-ratio}, yields that with probability at least $1-\delta$
    \begin{align*}
         \frac{\cT_{1,\cD}}{\cT_{1,\cD'}} &\le \exp\left(\frac{2L\sigma\sqrt{8\log(2/\delta)} + (2L)^2}{2\sigma^2}\right) \lep{a} \exp(\epsilon),
    \end{align*}
    where (a) holds if $\sigma \ge \frac{L \sqrt{8\log(2/\delta) + 4\epsilon}}{\epsilon}$. Combining this with $\cT_{2,\cD} / \cT_{2,\cD'} = 1$, yields the required result, hence finishing the proof.
\end{proof}

\subsection{Proof of Theorem~\ref{thm:central-upper}}
\label{proof:central-upper}
Before presenting the proof, we first introduce the following useful lemma.

\begin{lemma}[Theorem 5.1.1 in~\citet{tropp2015introduction}]
\label{lem:matrix-ch}
    Consider a finite sequence $\{X_i\}$ of independent random, symmetric matrices in $\Real^d$. Assume that $\lambda_{\min}(X_i) \ge 0$ and $\lambda_{\max}(X_i) \le H$ for each $i$. Let $Y = \sum_i X_i$ and $\mu_{\min}$ denote the minimum eigenvalue of the expectation $\ex{Y}$, i.e., $\mu_{\min} = \lambda_{\min}(\sum_i \ex{X_i})$. Then, for any $\epsilon \in (0,1)$, it holds 
    \begin{align*}
        \prob{\lambda_{\min}(Y) \le \epsilon \mu_{\min}} \le d \cdot \exp\left(-(1-\epsilon)^2 \frac{\mu_{\min}}{2H}\right).
    \end{align*}
\end{lemma}

Now, we are ready to prove Theorem~\ref{thm:central-upper}.
\begin{proof}[Proof of Theorem~\ref{thm:central-upper}]
    Let $ \widetilde{\cL}_{\cD}(\theta):= {l}_{\cD}(\theta) + \frac{\beta}{2n} \norm{\theta}_2^2 +  \frac{w^{\top} \theta}{n} $.  We divide the proof into the following steps.

    \textbf{Step 1:} Let $\Delta:= \hat{\theta}_{\texttt{obj}} - \theta^*$. Show that $c \norm{\Delta} \le \norm{\nabla \widetilde{\cL}_{\cD}(\theta^*)}$ for some positive constant $c$.

    To this end, note that we always have 
\begin{align*}
    \widetilde{\cL}_{\cD}(\theta^* + \Delta) -  \widetilde{\cL}_{\cD}(\theta^* ) - \inner{\nabla \widetilde{\cL}_{\cD}(\theta^*)}{\Delta} \le - \inner{\nabla \widetilde{\cL}_{\cD}(\theta^*)}{\Delta},
\end{align*}
since $\widetilde{\cL}_{\cD}(\theta^* + \Delta) = \widetilde{\cL}_{\cD}(\hat{\theta}_{\texttt{obj}}) \le \widetilde{\cL}_{\cD}(\theta^* )$ by the optimality of $\hat{\theta}_{\texttt{obj}}$. The RHS of above inequality can be upper bounded by $\norm{\nabla \widetilde{\cL}_{\cD}(\theta^*)}\norm{\Delta}$. Thus, it remains to lower bound the LHS, which motivates us to show that $\widetilde{\cL}_{\cD}$ is strongly convex with respect to the $\ell_2$-norm $\norm{\cdot}_{2}$. That is, we need to show that for all $v, \theta$, 
\begin{align*}
    v^{\top} \nabla^2 \widetilde{\cL}_{\cD}(\theta)v \ge c \norm{v}_{2}^2
\end{align*}
for some positive constant $c >0$. Now, let us look at the Hessian matrix of $\widetilde{\cL}_{\cD}$ at any $\theta$,
\begin{align}
\label{eq:interm}
    \nabla^2 \widetilde{\cL}_{\cD}(\theta) = \frac{\beta}{n} I + \frac{1}{n}\sum_{i=1}^n \sigma(x_i^{\top}\theta)(1-\sigma(x_i^{\top}\theta)) x_i x_i^{\top}. 
\end{align}

To proceed, we will leverage Lemma~\ref{lem:matrix-ch}. In particular, to apply it to our case, we have $X_i = x_i x_i^{\top}$ with $H = L^2$, $\mu_{\min} = n \kappa$ by Assumptions~\ref{ass:bound} and~\ref{ass:cov}.  Hence, as a result of Lemma~\ref{lem:matrix-ch}, with probability at least $1-\alpha$,
\begin{align}
\label{eq:emp-lambda}
    \lambda_{\min}(\sum_{i} x_i x_i^{\top}) \ge \frac{n\kappa}{2},
\end{align}
when $n \ge \frac{8L^2\log(d/\alpha)}{\kappa}$. Thus, condition on the good event,  plugging~\eqref{eq:emp-lambda} into~\eqref{eq:interm} and noting that $ \inf_{z \in [-2LB, 2LB]} \sigma(z) (1-\sigma(z)) \ge \gamma:= \frac{1}{2 + \exp(-2LB) + \exp(2LB)}$, yields that 
\begin{align*}
    v^{\top} \nabla^2 \widetilde{\cL}_{\cD}(\theta)v \ge  \left(\frac{\beta}{n} + \frac{\kappa \gamma}{2}\right) \norm{v}_{2}^2.
\end{align*}
Thus, we have so far established that 
\begin{align}
\label{eq:step1}
   \left(\frac{\kappa \gamma}{2}\right)\norm{\Delta} \le  \left(\frac{\beta}{n} + \frac{\kappa \gamma}{2}\right)\norm{\Delta} \le {\norm{\nabla \widetilde{\cL}_{\cD}(\theta^*)}},
\end{align}
where $\beta > 0$.

\textbf{Step 2:} Bound ${\norm{\nabla \widetilde{\cL}_{\cD}(\theta^*)}}$ with high probability.

To this end, we note that 
\begin{align*}
   {\nabla \widetilde{\cL}_{\cD}(\theta^*)} &= { \underbrace{\frac{1}{n} \sum_{i=1}^n \left( x_i (\sigma(x_i^{\top}\theta^*) - y_i)\right)}_{\cT_1} + \underbrace{\frac{\beta}{n} \theta^*}_{\cT_2} + \underbrace{\frac{w}{n}}_{\cT_3}}.
\end{align*}
Thus, we need to bound each of the terms on the RHS. We start with $\cT_1$. Let $V_i:= \sigma(x_i^{\top}\theta^*) - y_i$ and hence we have $\ex{V_i} = 0$. Thus, we can write $\cT_1 = -\frac{1}{n} X^T V$, where $X \in \Real^{n \times d}$ is the data matrix and $x_i^{\top} \in \Real^d$ is the $i$-th row of it. Hence, we have $\norm{\cT_1}^2 = \frac{1}{n^2} V^{\top} X X^{\top} V$. To analyze the concentration for this quadratic form, we will resort to the classic Hanson-Wright inequality. In particular, we will apply the explicit bound in Theorem 2.1 of~\citet{hsu2012tail}. To this end, we need to check the following quantities of $M:= \frac{1}{n^2}XX^{\top}$:
\begin{align*}
    &\tr(M) \le 4 L^2 / n\\
    &\tr(M^2) \le 16  L^4 / n^2\\
    &\norm{M}_{\text{op}} = \lambda_{\max}(M) \le 4 L^2 / n,
\end{align*}
where the above inequalities hold by simple linear algebra and the boundedness assumption. Thus, by  Theorem 2.1 of~\citet{hsu2012tail}, we have with probability at least $1-\alpha$
\begin{align*}
    \norm{\cT_1}^2 = V^{\top} M V \le C_1 L^2 \frac{1 + \log(1/\alpha)}{n},
\end{align*}
where $C_1$ is some universal constant.

For $\cT_2$, we have $\norm{\cT_2} \le \frac{\beta B}{n}$ by boundedness assumption. For $\cT_3$, by the standard concentration of the norm of Gaussian vector (cf. Theorem 3.1.1 in~\citet{vershynin2018high}), we have with probability at least $1-\alpha$, 
\begin{align*}
    \norm{\cT_3} \le C_3 \frac{1}{n} \sigma \left(\sqrt{d} + \sqrt{\log(1/\alpha)}\right),
\end{align*}
where $C_3$ is again form universal constant.

Putting all these bounds together and choosing $\beta = \sqrt{n}/B$, yields that with probability at least $1-\alpha$, 
\begin{align}
\label{eq:step2}
    {\norm{\nabla \widetilde{\cL}_{\cD}(\theta^*)}} \le C \cdot \left(L \sqrt{\frac{1+\log(1/\alpha)}{n}} +  \sigma \frac{1}{n} \sqrt{d} + \sigma  \frac{1}{n}\sqrt{\log(1/\alpha)}\right),
\end{align}
where $C$ is some universal constant.

\textbf{Step 3:} Derive the final bound. 

Plugging the bound in~\eqref{eq:step2} into~\eqref{eq:step1}, yields 
\begin{align*}
    \norm{\Delta} \le C' \left(\frac{L}{\kappa \gamma}\sqrt{\frac{1+\log(1/\alpha)}{n}} + \frac{\sigma \left(\sqrt{d} + \sqrt{\log(1/\alpha)}\right)}{n \kappa \gamma}\right).
\end{align*}
Recall that $\sigma =  \frac{L \sqrt{8\log(2/\delta) + 4\epsilon}}{\epsilon}$, and hence we finally have the bound
\begin{align*}
    \norm{\Delta} \le C' \left(\frac{L}{\kappa \gamma}\sqrt{\frac{1+\log(1/\alpha)}{n}} + \frac{\left(\sqrt{d} + \sqrt{\log(1/\alpha)}\right)}{n \kappa \gamma}\cdot \frac{L \sqrt{8\log(2/\delta) + 4\epsilon}}{\epsilon}\right).
\end{align*}
\end{proof}

\subsection{Semi-norm Error Bounds under Central Label DP}
\label{sec:central-semi}
In this section, we prove bounds on the estimation error in semi-norm under central label DP. 

\subsubsection{Lower Bound}
We have the following result for the lower bound on the estimation error. 
\begin{theorem}
\label{thm:central-lower-semi}
   For a large enough $n$, any estimator $\hat{\theta}$ based on samples form the BTL model that satisfies $(\epsilon,0)$-label DP in the central model has the estimation error in semi-norm lower bounded as
   \begin{align*}
   \ex{\norm{\hat{\theta} - \theta^*}^2_{\Sigma_{\cD} + \lambda I}} \ge \Omega\left(
    \frac{d}{n} +   \frac{d^2}{n^2 \epsilon^2}\right).
\end{align*}
\end{theorem}

To prove the theorem, we will leverage the following useful result, i.e., DP version of Fano's lemma\footnote{As before, we correct some constant factor errors in the original statement in~\cite{acharya2021differentially}.}. 
\begin{lemma}[Fano's lemma for central DP~\citep{acharya2021differentially}]
    Let $\cV = \{P_1, P_2,\ldots, P_M\} \subseteq \cP$ such that for all $i \neq j$,
    \begin{align*}
    \kl{P_i}{P_j} \le \beta~,\quad
        \rho'(\theta(P_i), \theta(P_j)) \ge \tau~.
        \end{align*}
        for a semi-metric $\rho'$ and for some $\tau,\beta > 0$.
        Moreover, let there exists a coupling between $P_i$ and $P_j$ such that $\ex{d_{\mathrm{ham}}(X,Y)} \le D$ for some $D>0$.
    Then, we have 
    \begin{align*}
        R(\cP, (\rho')^2, \epsilon) \ge \max\left\{ \frac{\tau^2}{4} \left(1 - \frac{\beta + \log 2}{\log M}\right), 0.2 \tau^2 \min\left\{1, \frac{M}{e^{10\epsilon D}}\right\}\right\}.
    \end{align*}
\end{lemma}
Now, we are ready to prove Theorem~\ref{thm:central-lower-semi}.
\begin{proof}[Proof of Theorem~\ref{thm:central-lower-semi}]
    The non-private part is the same as before, i.e., the proof for Theorem~\ref{thm:semi-lower}. For the private part, we follow the same packing construction as in the proof of Theorem~\ref{thm:semi-lower}. Let $(X, Y)$ be the coupling between $P_i^n$ and $P_j^n$, since $n$ samples are observed. Again, we utilize the maximal coupling property to obtain 
     \begin{align*}
        \ex{d_{\mathrm{Ham}}(X,Y)}  &= \sum_{k=1}^n \norm{P_{i,k} - P_{j,k}}_{\mathrm{TV}}\\
        &\le \sqrt{n} \sqrt{\sum_{k=1}^n \norm{P_{i,k} - P_{j,k}}_{\mathrm{TV}}^2}\\
        &\le \sqrt{n/2} \sqrt{\sum_{k} \kl{P_{i,k}}{P_{j,k}}}\\
        &\lep{a} \sqrt{n/2} \sqrt{n\Delta^2} \le 1/\sqrt{2} n \Delta := D,
    \end{align*}
    where (a) follows from~\eqref{eq:kl}.
    Now, noting that $\tau^2 = \Theta(\Delta^2)$, $M = \Theta(e^d)$, letting $\Delta = c\cdot \frac{d}{n\epsilon}$, we obtain 
    \begin{align*}
        R_c(\cP_{\log},\norm{\cdot}_{\Sigma_{\cD} + \lambda I}^2,\epsilon) \ge \Omega\left(\frac{d^2}{n^2 \epsilon^2}\right).
    \end{align*}
\end{proof}

\subsubsection{Upper Bound}
\begin{theorem}
\label{thm:central-upper-semi}
    Let $\alpha \in (0,1)$. Then, under Assumptions~\ref{ass:bound} and~\ref{ass:cov}, $\hat{\theta}_{\texttt{obj}}$ satisfies
     \begin{align*}
     \norm{\hat{\theta}_{\texttt{obj}} - \theta^{*}}_{\Sigma_{\cD} + \lambda' I} \le O\left( \frac{1}{\gamma} \sqrt{\frac{d + \log(1/\alpha)}{n}} + \frac{\sqrt{\sigma} (d\log(1/\alpha))^{1/4}}{\sqrt{n \gamma B}}
 \right)
     \end{align*}
      with probability at least $1-\alpha$,
     where $\gamma:= \frac{1}{2 + \exp(-2LB) + \exp(2LB)}$ and $\lambda' := \frac{\sigma \sqrt{d\log(1/\alpha)}}{\gamma n B}$. Thus, setting noise parameter $\sigma = \frac{L \sqrt{8\log(2/\delta) + 4\epsilon}}{\epsilon}$, it satisfies $(\epsilon,\delta)$-label DP in the central model and has estimation error 
     \begin{align*}
    \norm{\hat{\theta}_{\texttt{obj}} - \theta^{*} }_{\Sigma_{\cD} + \lambda' I} \le O\left( \frac{1}{\gamma} \sqrt{\frac{d + \log(1/\alpha)}{n}} + \frac{\sqrt{L} \left((\log(2/\delta) + 4\epsilon)d\log(1/\alpha)\right)^{1/4}}{\sqrt{n \epsilon \gamma B}}\right).
\end{align*}
\end{theorem}
\begin{proof}
Let $ \widetilde{\cL}_{\cD}(\theta):= {l}_{\cD}(\theta) + \frac{\beta}{2n} \norm{\theta}_2^2 +  \frac{w^{\top} \theta}{n} $. We divide the proof into the following steps.

\textbf{Step 1:} Let $\Delta:= \hat{\theta}_{\texttt{obj}} - \theta^*$. Show that $c \norm{\Delta}_{\Sigma_{\cD}}^2 \le \norm{\nabla \widetilde{\cL}_{\cD}(\theta^*)}_{(\Sigma_{\cD} + \lambda I)^{-1}} \norm{\Delta}_{\Sigma_{\cD} + \lambda I}$, for some positive constant $c$.

To this end, note that we always have 
\begin{align*}
    \widetilde{\cL}_{\cD}(\theta^* + \Delta) -  \widetilde{\cL}_{\cD}(\theta^* ) - \inner{\nabla \widetilde{\cL}_{\cD}(\theta^*)}{\Delta} \le - \inner{\nabla \widetilde{\cL}_{\cD}(\theta^*)}{\Delta},
\end{align*}
since $\widetilde{\cL}_{\cD}(\theta^* + \Delta) = \widetilde{\cL}_{\cD}(\hat{\theta}_{\texttt{obj}}) \le \widetilde{\cL}_{\cD}(\theta^* )$ by the optimality of $\hat{\theta}_{\texttt{obj}}$. The RHS of above inequality can be upper bounded by $\norm{\nabla \widetilde{\cL}_{\cD}(\theta^*)}_{(\Sigma_{\cD} + \lambda I)^{-1}} \norm{\Delta}_{\Sigma_{\cD} + \lambda I}$ for any $\lambda > 0$. Thus, it remains to lower bound the Hessian. That is, we aim to show that for all $v, \theta \in \Theta_B$, 
\begin{align*}
    v^{\top} \nabla^2 \widetilde{\cL}_{\cD}(\theta)v \ge c \norm{v}_{\Sigma_{\cD}}^2
\end{align*}
for some positive constant $c >0$. By definition, the Hessian of $\widetilde{\cL}_{\cD}$ at any $\theta \in \Theta_B$ is 
\begin{align*}
    \nabla^2 \widetilde{\cL}_{\cD}(\theta) &= \frac{\beta}{n} I + \frac{1}{n}\sum_{i=1}^n \sigma(x_i^{\top}\theta)(1-\sigma(x_i^{\top}\theta)) x_i x_i^{\top}\\
    &\ge \gamma \norm{v}_{\Sigma_{\cD} + \lambda' I}^2,
\end{align*}
where the inequality follows from $\beta >0$ and $\inf_{z \in [-2LB, 2LB]} \sigma(z) (1-\sigma(z)) \ge \gamma:= \frac{1}{2 + \exp(-2LB) + \exp(2LB)}$ and $\lambda' := \beta / (\gamma n)$  Thus, for all $\Delta$ such that $\theta^* + \Delta \in \Theta_B$, by Taylor expansion, we have 
\begin{align*}
    \frac{\gamma}{2} \norm{\Delta}_{\Sigma_{\cD} + \lambda' I}^2 \le \norm{\nabla \widetilde{\cL}_{\cD}(\theta^*)}_{(\Sigma_{\cD} + \lambda' I)^{-1}} \norm{\Delta}_{\Sigma_{\cD} + \lambda' I}.
\end{align*}
\textbf{Step 2:} Bound $\norm{\nabla \widetilde{\cL}_{\cD}(\theta^*)}_{(\Sigma_{\cD} + \lambda' I)^{-1}}$ with high probability.

To this end, we note that 
\begin{align*}
   {\nabla \widetilde{\cL}_{\cD}(\theta^*)} &= { \underbrace{\frac{1}{n} \sum_{i=1}^n \left( x_i (\sigma(x_i^{\top}\theta^*) - y_i)\right)}_{\cT_1} + \underbrace{\frac{\beta}{n} \theta^*}_{\cT_2} + \underbrace{\frac{w}{n}}_{\cT_3}}.
\end{align*}
By the same analysis as in~\citet{zhu2023principled}, we have with probability at least $1-\alpha$
\begin{align*}
    \norm{\cT_1}_{(\Sigma_{\cD} + \lambda' I)^{-1}} \le C \cdot  \sqrt{\frac{d + \log(1/\alpha)}{n}},
\end{align*}
for some absolute constant $C$.
For $\cT_2$, we have
\begin{align*}
      \norm{\cT_2}_{(\Sigma_{\cD} + \lambda' I)^{-1}} \le \frac{\beta}{n\sqrt{\lambda'}}\norm{\theta^*} \le \frac{\beta B}{n \sqrt{\lambda'}}.
\end{align*}
For $\cT_3$, by the concentration of the norm of the Gaussian vector, we have with probability at least $1-\alpha$,
\begin{align*}
    \norm{\cT_3}_{(\Sigma_{\cD} + \lambda' I)^{-1}} \le \frac{1}{n\sqrt{\lambda'}} \norm{w} 
    \le O\left(\frac{\sigma}{n\sqrt{\lambda'}} \left(\sqrt{d} + \sqrt{\log(1/\alpha)}\right)\right).
\end{align*}
Putting all of them together, we have 
\begin{align*}
     \frac{\gamma}{2} \norm{\Delta}_{\Sigma_{\cD} + \lambda' I}^2 \le C' \left(\sqrt{\frac{d + \log(1/\alpha)}{n}} + \frac{\beta B}{n\sqrt{\lambda'}} + \frac{\sigma \sqrt{d\log(1/\alpha)} }{n\sqrt{\lambda'}} \right) \norm{\Delta}_{\Sigma_{\cD} + \lambda' I},
\end{align*}
which directly implies that 
\begin{align*}
    \norm{\Delta}_{\Sigma_{\cD} + \lambda' I} \le C_1 \cdot \frac{1}{\gamma} \sqrt{\frac{d + \log(1/\alpha)}{n}} +  C_1 \frac{1}{\gamma} \left( \frac{\beta B}{n\sqrt{\lambda'}} + \frac{\sigma \sqrt{d\log(1/\alpha)} }{n\sqrt{\lambda'}}\right).
\end{align*}
Thus, choosing $\lambda' = \frac{\sigma \sqrt{d\log(1/\alpha)}}{\gamma n B}$ (i.e., $\beta = \frac{\sigma \sqrt{d\log(1/\alpha)}}{B}$), yields that 
\begin{align*}
     \norm{\Delta}_{\Sigma_{\cD} + \lambda' I} \le O\left( \frac{1}{\gamma} \sqrt{\frac{d + \log(1/\alpha)}{n}} + \frac{\sqrt{\sigma} (d\log(1/\alpha))^{1/4}}{\sqrt{n \gamma B}}
 \right).
\end{align*}
Finally, plugging in noise value $\sigma = \frac{L \sqrt{8\log(2/\delta) + 4\epsilon}}{\epsilon}$, yileds our final result 
\begin{align*}
    \norm{\hat{\theta}_{\texttt{obj}} - \theta^{*} }_{\Sigma_{\cD} + \lambda' I} \le O\left( \frac{1}{\gamma} \sqrt{\frac{d + \log(1/\alpha)}{n}} + \frac{\sqrt{L} \left((\log(2/\delta) + 4\epsilon)d\log(1/\alpha)\right)^{1/4}}{\sqrt{n \epsilon \gamma B}}\right).
\end{align*}
Note that the privacy guarantee follows the same as before, hence completing the proof.
\end{proof}

\section{Generalization to Standard DP}
\label{app:stdDP}

In the main paper, we mainly focus on protecting the labels via label DP, which is well-motivated by many practical situations. It turns out that our technique can also be generalized to handle privacy protection of both features and labels, i.e., the standard DP notion. 

We start with the central model. Since objective perturbation~\citep{kifer2012private} was originally proposed to achieve standard DP in the central model, it would be natural to adopt it in our case. However, as before, we cannot directly employ the results in~\cite{kifer2012private} to prove privacy guarantee due to the gap in their Lemma 16. Instead, we found that for log loss, one can get rid of the independence issue in their Lemma 16, and hence establish the privacy guarantee with the same order of Gaussian noise. This is not true in general for arbitrary convex losses (where an additional $\sqrt{d}$ factor is required), as also observed in~\cite{agarwal2023differentially}. 

\textbf{Privacy.} In the following, we will show that with minor constant changes in the noise parameter of Theorem~\ref{thm:central-priv}, Algorithm~\ref{alg:objP} also achieves standard DP in the central model, i.e., the neighboring relation is now about a change of $(x_i, y_i)$ rather than only $y_i$ under label DP as considered in Theorem~\ref{thm:central-priv}.

\begin{theorem}[Privacy under standard DP]
\label{thm:central-priv-dp}
     Let $\epsilon \!>\!0$, $\delta \!\in\! (0,1)$ and Assumption~\ref{ass:bound} hold. Then, setting $\sigma \ge \frac{4L \sqrt{8\log(4/\delta) + 2\epsilon}}{\epsilon}$ and $\beta \ge \frac{4L^2}{\epsilon}$, 
Algorithm~\ref{alg:objP} satisfies $(\epsilon,\delta)$-DP in the central model.
\end{theorem}
\begin{proof}
    As in the proof of Theorem~\ref{thm:central-priv}, we will use two Successive Approximations, which allows us to only focus on the following problem
     \begin{align*}
        \widetilde{\theta}_{\cD} = \argmin_{\theta \in \Real^d} {l}_{\cD}(\theta) + \frac{\beta}{2n} \norm{\theta}_2^2 +  \frac{w^{\top} \theta}{n} + \frac{1}{n}r(\theta).
    \end{align*}
    Also, as before, we have the following PDF. 
    \begin{align}
    \label{eq:pdf-dp}
        f_{\widetilde{\theta}_{\cD}}(t) = C \underbrace{\exp\left( - \frac{\norm{\psi_{\cD}(t)}_2^2}{2\sigma^2}\right)}_{\cT_{1,\cD}}\cdot \underbrace{\left| \det \left[ \frac{d \psi_{\cD} (\theta) }{d \theta} |_{\theta = t} \right]\right|}_{\cT_{2,\cD}}
    \end{align}
    We are again left to bound the two ratios. To this end, we first look at the ratio of $\cT_{2,\cD} / \cT_{2,\cD'}$. Note that the matrix inside the determinant in~\eqref{eq:pdf-dp} is the Hessian of $l_{\cD}$ plus some common terms, given by 
    \begin{align*}
        A_{\cD} := \nabla^2 r(\theta)|_{\theta = t} +   {\beta} I + \nabla^2 {l}_{\cD} (\theta) |_{\theta = t} =  \nabla^2 r(\theta)|_{\theta = t} + {\beta} I + \sum_{i=1}^n \sigma(x_i^{\top}t)(1-\sigma(x_i^{\top}t)) x_i x_i^{\top},
    \end{align*}
    which now depends on $x_i$. Hence, we need some additional steps to bound this ratio under standard DP. In particular, we define $E:= A_{\cD} - A_{\cD'} = \sigma(x_s^{\top}t)(1-\sigma(x_s^{\top}t)) x_s x_s^{\top} - \sigma(x_s'^{\top}t)(1-\sigma(x_s'^{\top}t)) x_s' x_s'^{\top}$, where $\cD, \cD'$ differs in one single sample at index $s$. Thus, the rank of $E$ is most two and moreover the sum of largest and second largest eigenvalue of $E$ satisfies
    \begin{align*}
        |\lambda_1(E)| + |\lambda_2(E)| \le \frac{1}{4} \cdot 4L^2 + \frac{1}{4} \cdot 4L^2 = 2L^2,
    \end{align*}
    where we have used the boundedness assumption. This also implies that 
    \begin{align*}
        |\lambda_1(E)| \cdot  |\lambda_2(E)| \le L^4. 
    \end{align*}
    To proceed, we will leverage the following result. 
    \begin{claim}[Lemma 10 in~\cite{chaudhuri2011differentially}]
    \label{clm:det}
        If $A$ is full rank and if $E$ has rank at most $2$, then 
        \begin{align*}
            \frac{\det(A+E) - \det(A)}{\det(A)} = \lambda_1(A^{-1}E) + \lambda_2(A^{-1}E) + \lambda_1(A^{-1}E) \cdot \lambda_2(A^{-1}E),
        \end{align*}
        where $\lambda_j(Z)$ is the $j$-th largest eigenvalue of matrix $Z$.
    \end{claim}
    Note that for $j = 1,2$, $|\lambda_j(A_{\cD'}^{-1}E)| \le \frac{|\lambda_j(E)|}{\beta}$ due to the fact that the minimal eigenvalue of $A_{\cD'}$ is at least $\beta$. Thus, by Claim~\ref{clm:det}, we have
    \begin{align*}
        \frac{\cT_{2,\cD}}{\cT_{2,D'}} = \frac{|\det(A_{\cD'} + E)|}{|\det(A_{\cD'})|} &= \left|1 + \lambda_1(A_{\cD'}^{-1}E) + \lambda_2(A_{\cD'}^{-1}E) + \lambda_1(A_{\cD'}^{-1}E) \cdot \lambda_2(A_{\cD'}^{-1}E)\right| \\
        &\le 1+ \frac{2L^2}{\beta} + \frac{L^4}{\beta^2}\\
        & =\left(1 + \frac{L^2}{\beta}\right)^2\\
        &\le e^{2L^2/\beta}.
    \end{align*}
    Thus, when $\beta \ge \frac{4L^2}{\epsilon}$, we have $ \frac{\cT_{2,\cD}}{\cT_{2,D'}} \le e^{\epsilon/2}$. 

    Now, let us turn to bound $ \frac{\cT_{1,\cD}}{\cT_{1,D'}}$. In particular, we have 
    \begin{align}
    \label{eq:T1-ratio-dp}
        \frac{\cT_{1,\cD}}{\cT_{1,\cD'}} &= \exp\left(\frac{\norm{\psi_{\cD'}(t)}_2^2 - \norm{\psi_{\cD}(t)}_2^2}{2\sigma^2}\right)\nonumber\\
        &= \exp\left(\frac{2\inner{\psi_{\cD}(t)}{\psi_{\cD'(t)} - \phi_{\cD}(t)} + \norm{\psi_{\cD'}(t) - \psi_{\cD}(t)}_2^2}{2\sigma^2}\right)\nonumber\\
         &= \exp\left(\frac{2\inner{-w}{\psi_{\cD'(t)} - \psi_{\cD}(t)} + \norm{\psi_{\cD'}(t) - \psi_{\cD}(t)}_2^2}{2\sigma^2}\right).
    \end{align}
    where we know that $\psi_{\cD}(t) = -w$, which is distributed according to a normal. However, one needs to be careful here to show that  $\psi_{\cD'}(t) - \psi_{\cD}(t)$ is \emph{independent} of $w$ so that one can claim that the inner product is also distributed according to a normal.  In our case, for two neighboring datasets $\cD, \cD'$ that differs only in $(x_s, y_s)$ and $(x_s', y_s')$  we have 
    \begin{align*}
        \psi_{\cD'}(t) - \psi_{\cD}(t) &=  \left(\frac{1}{1+\exp(-\inner{x_s'}{t})} - y_s'\right) x_s' - \left(\frac{1}{1+\exp(-\inner{x_s}{t})} - y_s\right) x_s.
    \end{align*}    
    Then, we have 
    \begin{align*}
         |\inner{-w}{\psi_{\cD'}(t) - \psi_{\cD}(t)}| \le |\inner{w}{x_s'}| + |\inner{w}{x_s}|,
    \end{align*}
    which combined with the concentration of normal distribution and boundedness of $x_s, x_s'$, leads to that with probability at least $1-\delta$,
    \begin{align*}
         |\inner{-w}{\phi_{\cD'}(t) - \phi_{\cD}(t)}| \le 4L \sigma \sqrt{2\log(4/\delta)}.
    \end{align*}
    Meanwhile, we have $\norm{\psi_{\cD'}(t) - \psi_{\cD}(t)} \le 4L$ by Assumption~\ref{ass:bound}
    
    Putting everything back to~\eqref{eq:T1-ratio-dp}, yields that with probability at least $1-\delta$
    \begin{align*}
        \frac{\cT_{1,\cD}}{\cT_{1,\cD'}} &\le \exp\left(\frac{2L\sigma\sqrt{8\log(4/\delta)} + (4L)^2}{2\sigma^2}\right) \lep{a} \exp(\epsilon/2),
    \end{align*}
    where (a) holds if $\sigma \ge \frac{4L \sqrt{8\log(4/\delta) + 2\epsilon}}{\epsilon}$. Combining this with $\cT_{2,\cD} / \cT_{2,\cD'} = e^{\epsilon/2}$, yields the required result, hence finishing the proof.\end{proof}

    \textbf{Utility.} For the estimation error under $\ell_2$ norm, one can follow the same proof of Theorem~\ref{thm:central-upper}. One difference is to remember to check the condition of $\beta$ in Theorem~\ref{thm:central-priv-dp}, which can be satisfied by conditions on $n$ and $\epsilon$. For the estimation error in semi-norm, one needs additional steps compared to the proof of Theorem~\ref{thm:central-upper-semi}, since now it needs to establish the concentration of $\norm{\cdot}_{\widetilde{\Sigma}_{\cD} + \lambda I}$, where $\widetilde{\Sigma}_{\cD}$ is the private covariance matrix. First, one can privatize the covariance matrix $\Sigma_{\cD}$ via the standard Gaussian mechanism. Then, to guarantee a semi-positive nature of $\widetilde{\Sigma}_{\cD} + \lambda I$, one needs to choose $\lambda$ properly, which can be done by following the routine in previous DP linear bandits (see~\cite{shariff2018differentially,pmlr-v162-chowdhury22a}). Finally, one can translate the concentration $\norm{\cdot}_{\Sigma_{\cD} + \lambda I}$ in Theorem~\ref{thm:central-upper-semi} to $\norm{\cdot}_{\widetilde{\Sigma}_{\cD} + \lambda I}$ in Theorem~\ref{thm:central-upper-semi} via standard Gaussian concentration and the property of linear summation of Gaussian. Ignoring all other factors ($\gamma$, $B$, $L$), the final cost of privacy should be on the order of $\frac{(d\log(1/\delta)^{1/4}}{\sqrt{n\epsilon}}$ for $\epsilon \in (0,1)$. One subtlety again is that one needs to check $\beta$ satisfies the condition in Theorem~\ref{thm:central-priv-dp}.

    \begin{remark}[Remark on local DP]
        In contrast to local label DP in the main paper, establishing local standard DP is challenging in our offline reward estimation setting, which is \emph{non-interactive}. This is different from interactive online logistic regression in~\cite{duchi2018minimax}.
        In fact, it is in general not straightforward to derive an efficient algorithm even for ERM under the non-interactive setting~\cite{smith2017interaction}, let alone the parameter estimation problem in our setting. We leave it to one of our future research directions. 
    \end{remark}

\end{document}